\documentclass{article}

\usepackage[nonatbib,preprint]{somestyle} % For ArXiv

\usepackage[utf8]{inputenc}
\usepackage[T1]{fontenc}    % use 8-bit T1 fonts
\usepackage{amsmath,amssymb,amsthm,amsfonts}
\usepackage{bm}
\usepackage[dvipsnames]{xcolor}

\usepackage{url}
\usepackage[ruled,noend]{algorithm2e}
\usepackage{mathtools}
\usepackage{multirow}
\usepackage{multicol}
\usepackage{wrapfig}

\usepackage[natbib=true]{biblatex}
\AtEveryBibitem{\clearfield{note}} % comment it out to see links in references

\usepackage{tikz}
\usetikzlibrary{calc}
\usetikzlibrary{intersections}

\usepackage{hyperref}       % hyperlinks
\usepackage{booktabs}       % professional-quality tables
\usepackage{amsfonts}       % blackboard math symbols
\usepackage{nicefrac}       % compact symbols for 1/2, etc.
\usepackage{microtype}      % microtypography
\usepackage{authblk}

\newcommand{\E}{\mathbb{E}}
\renewcommand{\Pr}{\mathbb{P}}

\newcommand{\distset}{\mathcal{P}}
\newcommand{\costbound}{\bar{c}}
\newcommand{\TD}{\Delta}
\newcommand{\armset}{\mathcal{X}}

\newcommand{\bigO}{\mathcal{O}}

\newcommand{\Fancyname}{Score Function-based Successive Reject}
\newcommand{\fcnm}{\texttt{SFSR}} 
\newcommand{\fcnmL}{\texttt{SFSR-L}}

\theoremstyle{definition}
\newtheorem{assumption}{Assumption}
\newtheorem{definition}{Definition}

\theoremstyle{remark}
\newtheorem{remark}{Remark}

\theoremstyle{plain}
\newtheorem{theorem}{Theorem}
\newtheorem{lemma}{Lemma}

\newtheorem{proposition}{Proposition}

\allowdisplaybreaks
\title{Pure Exploration for Constrained Best Mixed Arm Identification with a Fixed Budget}
\author[1]{\bf Dengwang Tang}
\author[1,*]{\bf Rahul Jain} 
\author[1]{\bf Ashutosh Nayyar}
\author[1]{\bf Pierluigi Nuzzo}

\affil[1]{Ming Hsieh Department of Electrical and Computer Engineering, University of Southern California, Los Angeles, CA 90089}
\affil[]{\texttt{dwtang@umich.edu,\{rahul.jain, ashutosn, nuzzo\}@usc.edu }}
\affil[*]{RJ is also affiliated with Google DeepMind. This work has been done at USC.}

\begin{document}

\maketitle

\begin{abstract}
In this paper, we introduce the constrained best mixed arm identification (CBMAI) problem with a fixed budget. This is a pure exploration problem in a stochastic finite armed bandit model. Each arm is associated with a reward and multiple types of costs from unknown distributions. Unlike the  unconstrained best arm identification problem, the optimal solution for the CBMAI problem may be a randomized mixture of multiple arms. The goal thus is to find the best mixed arm that maximizes the expected reward subject to constraints on the expected costs with a given learning budget $N$. We propose a novel, parameter-free algorithm, called the Score Function-based Successive Reject (SFSR) algorithm, that combines the classical successive reject framework with a novel score-function-based rejection criteria based on linear programming theory to identify the optimal support. We provide a theoretical upper bound on the mis-identification (of the the support of the best mixed arm) probability and show that it decays exponentially in the budget $N$ and some constants that characterize the hardness of the problem instance. We also develop an information theoretic lower bound on the error probability that shows that these constants appropriately characterize the problem difficulty. We validate this empirically on a number of average and hard instances.
\end{abstract}

\section{Introduction}\label{sec:intro}
Bandit models are prototypical models of online learning, exploration and decision making \citep{lattimore2020bandit}. For example, recommender systems for online shopping, video streaming, etc. often use learning algorithms to make recommendations that maximize click-through-rates. Online learning  can either be formulated as a regret minimization problem which leads to algorithms that tradeoff exploration and exploitation \citep{lai1985asymptotically}, where the regret of a learning algorithm is defined with respect to a policy optimizing a single (reward) objective. A number of Bayesian and non-Bayesian algorithms \cite{lattimore2020bandit} have  been proposed for this setting including some that tradeoff the exploration and the exploitation close to the information-theoretic limit possible \cite{auer2002finite,russo-info-ratio}. Or, they can be formulated as a pure exploration problem, also referred to as the best arm identification problem  \cite{audibert2010best} wherein we do adaptive data collection either with a fixed budget \cite{audibert2010best}, or for a fixed confidence \cite{garivier2016optimal}, with the goal of minimizing the probability of mis-identification of the optimal arm. A range of algorithms have been proposed that achieve this goal \cite{audibert2010best,karnin2013almost,garivier2016optimal,qin2017improving}.

However, many online learning problems involve multiple objectives that cannot be aggregated into a single objective function. It is best to formulate such problems as maximizing one objective while constraining the others. In recent literature, progress has been made on constrained bandit models, as well as online reinforcement learning with constraints. In such exploration vs. exploitation settings, surprisingly there are now algorithms that minimize regret while ensuring bounded constraint violation \citep{kalagarla2023safe}. However, sometimes online learning problems are better formulated as pure exploration problems. For example, in video recommendation systems, while it is feasible to allow for some experimentation and exploration to learn user preferences for a limited time, learn an optimal strategy and then freeze it (as for best arm identification), it is simply not practically feasible to allow for continual adaptation for exploration (as for regret minimization). And yet, many practical best arm identification problems involve multiple objectives, which are often aggregated unnaturally into a single objective primarily because the problem of constrained best arm identification is unsolved. 

In this paper, we introduce the \textit{constrained mixed best arm identification} (CBMAI) problem wherein there are $K$ arms, each of which is associated with a reward and multiple cost attributes. These are random, and come from distributions with unknown means. Given a sampling budget $N$, we can do pure exploration and sample the arms in any way, and at the end must identify a mixed arm (i.e., a randomization over a subset of the deterministic arms) such that it meets each of the cost constraints. We allow for mixed arms because they may do better than any deterministic arm under average cost constraints. Our objective is to design an algorithm that minimizes the mis-identification error. 

\textbf{Our Algorithm and Contributions.} %[\tdwcomment{Work on it once experimental results are settled.}]
(i) We provide the first algorithm, called the \Fancyname{} (\fcnm{}) Algorithm, that identifies the optimal support for the CBMAI problem under the fixed-budget setting. The algorithm combines the successful Successive Reject algorithm \citep{audibert2010best} with a novel rejection criteria (using a score function) based on linear programming theory. It addresses the key difficulty that under unknown costs there are uncountably many candidates for the best mixed arm by focusing on  \textit{optimal support identification}.  Using different score functions results in an algorithm with different flavors, and we present two choices for them. (ii) We establish a performance guarantee of the proposed algorithm in the form of an instance-dependent error upper bound, which decays exponentially in $N$ with an exponent characterized by a certain measure of hardness of the instance. This is validated by a lower bound on the error probability over a broad class of algorithms. (iii) We provide empirical results to show that our proposed algorithm significantly outperforms baselines on typical instances.

\textbf{Related Literature.} There is vast literature on multi-arm bandit models summarized well in \cite{lattimore2020bandit}.  We review related literature on best arm identification and constrained learning problems. 

\textbf{Best Arm Identification.} The literature on (unconstrained) best arm identification can be divided into two categories: (1) The fixed confidence setting \citep{kaufmann2013information,jamieson2014lil, garivier2016optimal,russo2016simple,qin2017improving}, where the objective is to identify the best arm with a specified error probability $\delta$ with the smallest amount of samples possible.  Two major algorithm design philosophies in this setting are Top-2 algorithms \citep{russo2016simple} and Track-and-Stop \citep{kaufmann2013information}. (2) The fixed budget setting \citep{audibert2010best,karnin2013almost,carpentier2016tight,yang2022minimax,barrier2023best,poanwang2024best}, where one aims to minimize the identification error given a learning budget $N$. In this setting, round based elimination algorithms are by far the dominant algorithm design philosophy. 
%Despite the two settings being seemly very close to each other, it has been established that for optimal performance, the two settings requires different techniques, and the information theoretic limits of the two settings are different in general \citep{kaufmann2016complexity}. 

\textbf{Regret-focused Learning in Constrained Problems.} 
There have been a lot of recent literature on designing algorithms to achieve small reward and/or constraint violation regret in constrained multi-armed bandits \citep{amani2019linear,moradipari2021safe,liuxin2021efficient,zhou2022kernelized,pacchiano2024contextual} and constrained MDPs \citep{liutao2021learning,bura2022dope,kalagarla2023safe}. 
%These algorithms are mostly adapted from the classical algorithms for unconstrained regret minimization such as UCB \citep{auer2002finite} or posterior sampling \citep{thompson1933likelihood}. 
In these algorithms, it is necessary to balance exploration and exploitation. In contrast, we are interested in efficiently using the learning budget for exploring without concerning the reward and cost accumulated in the process. 

\textbf{Constrained Best Arm Identification.}
Recently, there have been considerable interest in best arm identification in a constrained setting. In \cite{lindner2022interactively,camilleri2022active,zhenlinwang2022best,faizal2022constrained,shang2023price}, the authors considered the problem of finding the best \emph{deterministic} arm out of a finite set of arms in either fixed-confidence or fixed-budget settings in various constrained multi-armed bandit settings. In contrast, our work focuses on finding the best \textit{mixed} arm. %In addition to the fact that the best mixed arm can differ from the best deterministic arm, the task of determining the best deterministic arm is vastly different from the one for the best mixed arm in two aspects: (i) the best deterministic arm is not necessarily in the support of the best mixed arm, and (for problems with two or more constraints) the support of the best mixed arm may not even contain arms that are feasible on their own; and (ii) an arm deemed violating the constraint can be ignored and eliminated in best deterministic arm identification, but it may not be for best mixed arm identification.
%To the best of our knowledge, we know of only three existing work related to constrained best mixed arm identification or its generalizations. 
There are very few works on best mixed arm identification: The CBMAI problem under the \emph{fixed-confidence} setting assuming that the \emph{costs are known}  was considered in \cite{carlsson2023pure}. In contrast, we assume that the costs are unknown. A fixed-budget best knapsack identification problem assuming that the best solution belongs to a known finite set, and there's an offline oracle for finding the best knapsack under constraints given an input reward function was considered in \cite{nakamura2024fixed}. In contrast, there are uncountably many candidates for the best mixed arm in our setting, and we do not assume access to such an oracle due to costs being unknown. A fixed-budget optimal support identification problem with the same constraints imposed on both the exploration process and the final solution was considered in \cite{li2023optimalarmsknapsacks}. %They proposed two algorithms based on a round based elimination algorithm based on the estimated Lagrangian function (which is also used in the PDSR algorithm in Section \ref{sec:pdsr}).  
In contrast, we do not impose any constraints on the exploration process. Furthermore, the theoretical error bounds therein are unfortunately not correct since the strong concentration results for optimal solutions of randomly perturbed linear programs in their Lemmas B.2 and C.2, which are a critical part of the theoretical analysis, are erroneous. 
%\footnote{In Lemma B.2 and Lemma C.2 of \cite{li2023optimalarmsknapsacks}, the authors made the following mistake: concluding $w^T c \approx \Bar{w}^T c$ from $w^T \bm{1} \approx \Bar{w}^T \bm{1}$ when $c$ is a generic vector. Such a mistake lead the authors to very strong concentration results for the solution of a noisy version of a linear program, which ultimately played a critical role of in the very strong error bound of both their algorithms.} 
%\textbf{Other related work:}
In a related line of research, \cite{kone2023bandit} considered the problem of Pareto front identification for arms with multiple attributes. There has also been numerous works on constrained Bayesian Optimization \cite{gardner2014bayesian,gelbart2014bayesian,letham2019constrained,eriksson2021scalable}, where the primary focus is on empirical performance instead of theoretical guarantees.

\paragraph{Notation} For positive integer $M$ we write $[M]:=\{1,2,\cdots, M\}$. For a vector $v\in\mathbb{R}^M$ and $\mathcal{I}\subset [M]$, we use $v_{\mathcal{I}}$ to denote the subvector of $v$ formed by indices in $\mathcal{I}$. For a matrix $A$ with $M$ columns and $\mathcal{I}\subset [M]$, we use $A_{\mathcal{I}}$ to denote the submatrix of $A$ formed by columns indexed by $\mathcal{I}$. We use $\mathcal{N}(\mu, \sigma^2)$ to represent a Gaussian distribution with mean $\mu$ and variance $\sigma^2$. %For vectors and matrices, $\|\cdot\|_2$ represents the Euclidean 2-norm. %\RJ{May be we can save space and introduce notation where needed.}

\section{Preliminaries}\label{sec:prelim}

\noindent\textbf{Problem Statement.}
We introduce the \emph{constrained best mixed arm identification} (CBMAI) problem in the context of a bandit model: There are $K$ arms, indexed by $[K] = \{ 1, 2, \cdots, K\}$, each associated with a reward function $R_a$ and $L$ cost functions $C_{l,a}$, $a \in [K], l \in [L]$. We would like to determine a mixed arm $p^*$, i.e., a probability distribution over the arms (in the probability simplex $\distset_K := \{ p\in \mathbb{R}_+^K: \bm{1}^T p = 1 \} $) such that it achieves the following
\begin{equation}\label{eq:LP1}
    \begin{split}
        \max_{p\in \distset_K}\{\mathbf{R}^T p : \mathbf{C} p \leq \costbound \},
    \end{split}
\end{equation}
where $\mathbf{R}=(R_1,\cdots, R_K)^T$, $(\mathbf{C})_{l,a} = C_{l,a}$ and the vector $\costbound \in\mathbb{R}^L$. This is a linear program, and by theory of linear programming \cite{luenberger1984linear}, an optimal solution of \eqref{eq:LP1} can be obtained on an extreme point of the constraint polytope. 
Thus, when the costs are known, the best mixed arm will lie in a known finite set \citep{carlsson2023pure,nakamura2024fixed} since the constraint polytope has a finite number of vertices. 

However, our motivation comes from the bandit setting, and typically both the rewards $R$ and costs $C$ for each arm are random, and come from an unknown distribution. In that case, there can be uncountably many candidates for the best mixed arm making identifying the exact best mixed arm virtually impossible. Thus, we focus  on the \emph{optimal support identification}, i.e., identifying the arms that have non-zero probability in the best mixed arm. We denote such a set of arms by $\mathcal{I}^*$ (we will define it precisely later).  We will consider that when the learning agent chooses arm $a$, it receives a reward $R_a \sim \mathcal{N}(r_a, \sigma_r^2)$, and also incurs costs $C_{l,a} \sim \mathcal{N}(c_{l,a}, \sigma_c^2)$,  $l=1,2,\cdots,L$. The random rewards and costs are assumed mutually independent. We will assume that for the first $K_0$ arms, the mean reward $(r_a)_{a\in [K_0]}$ and mean costs $(c_{l,a})_{a\in[K_0], l \in [L]}$ are unknown. The means of the reward and costs of arm $a\in \{K_0+1, \cdots, K\}$ is assumed to be known. The variances are not needed to be known by the algorithms we design, but assuming them known will simplify our analysis. %Furthermore, if needed to model outside options, one can assume means of some arms as being known but this will not affect our algorithms, or their analysis. 
Thus, we would like to solve the following LP problem that optimizes the expected reward subject to constraints on expected costs
\begin{equation}\label{eq:LP}
    \begin{split}
        \max_{p\in \distset_K }\{\mathbf{r}^T p : \mathbf{c} p \leq \costbound \},
    \end{split}\tag{LP}
\end{equation}
where the components of $\mathbf{r}=(r_a)_{a\in [K]}$ and $\mathbf{c} = (c_{l, a})_{l\in[L], a\in [K]}$ corresponding to the first $K_0$ arms are unknown and must be learnt.  

To that end, we need samples of reward and costs for various arms. We consider a \textit{fixed-budget} setting, i.e., we can only obtain \textit{at most} $N$ such samples. 
We assume an underlying probability space $(\Omega, \mathcal{F}, \mathbb{P})$, and would like to design a learning agent $\phi$ that minimizes the \emph{misidentification probability} $\Pr_{\mathbf{r},\mathbf{c}}(\armset^{\mathrm{\phi}} \neq \mathcal{I}^*)$,
i.e. the probability of misidentifying the optimal support $\mathcal{I}^*$, where $\armset^{\mathrm{\phi}}$ is the subset of arms output by the algorithm. Note that we use a \textit{strict criteria}: we only consider the identification to be correct if the output support is \emph{exactly} the set of arms in the optimal support.

\begin{remark} We make a few observations.
(i) Our algorithm does not need the Gaussian distributions assumption. Furthermore, our main result (Theorem \ref{thm:main}) can be adapted to sub-Gaussian reward and cost distributions. However, assuming Gaussian distributions allows us to focus on presenting key ideas while reducing unnecessary technical details related to concentration inequalities that can obfuscate the intuition.
(ii) For best \textit{deterministic arm} identification problem, constrained or not, explicitly formulating known arms in the model may not be necessary: One can always run any algorithm on the subset of unknown arms, obtain the best arm in this subset, and compare it with the best known arm. However, this is not the case for constrained best mixed arm identification, as the addition of a known arm could introduce new unknown arms into the optimal mixed arm. 
(iii) If the optimal support is known, one can easily finetune the mixing probabilities with online data and quickly converge  to the best mixed arm without the need to explore arms outside of the support. However, finding the optimal support can be a challenging problem due to its combinatorial nature. 
%As we discussed in the introduction, optimal support identification is arguably much more important than finding the exact mixing probabilities: During online deployment, one can always use the online data (of costs) to finetune the best mixing probabilities and eventually converges to the best mixed arm without exploring any other arms. However, if a certain arm $i$ in the optimal support is incorrectly ruled out by the algorithm, then linear regret (in either reward or constraint violation) would be inevitable in the online application if the algorithm does not pull arm $i$.
\end{remark}

\section{The \Fancyname{} (\fcnm) Algorithm}\label{sec:ivsr}

We first derive our main algorithm, the \emph{\Fancyname} (\fcnm) algorithm, that uses a novel elimination rule we designed based on the \emph{intersection value} (IV) score. We will later show that substituting this score function with another results in a different flavor of the algorithm that can also have good empirical performance.

Consider the standard form of \eqref{eq:LP} where we add the slack vector $s\in\mathbb{R}_+^L$:
\begin{equation}\label{eq:SFLP0}
    \begin{split}
        \max_{p\in\mathbb{R}_+^K, s\in\mathbb{R}_+^L} \{\mathbf{r}^T p : \mathbf{c} p + s = \costbound, \bm{1}^T p = 1\}
    \end{split}
\end{equation}

Let $\bm{0}_m$ (resp. $\bm{1}_m$) denote the all-0 vector (resp. all-one vector) of length $m$. Let $\mathbf{I}_{L\times L}$ denote the $L$-by-$L$ identity matrix. Set
\begin{align}
    \mathbf{x} = \begin{pmatrix}
        p\\s
    \end{pmatrix},\quad    
    \mu = \begin{pmatrix}
    \mathbf{r}\\\bm{0}_{L}
    \end{pmatrix},\quad 
    \mathbf{A} = \begin{pmatrix}
        \mathbf{c}&\mathbf{I}_{L\times L}\\ \bm{1}_K^T& \bm{0}_L^T
    \end{pmatrix},\quad 
    \mathbf{b} = \begin{pmatrix}
        \costbound \\ 1
    \end{pmatrix},
\end{align}
then \eqref{eq:SFLP0} can be simply written as 
\begin{equation}\label{eq:SFLP}
    \max\{\mu^T \mathbf{x} : \mathbf{A}\mathbf{x} = \mathbf{b}, \mathbf{x}\geq 0 \}\tag{SFLP}
\end{equation}
By the fundamental theorem of linear programming \citep{luenberger1984linear}, an optimal solution of \eqref{eq:SFLP} can be obtained at a \emph{basic feasible solution} (BFS) of $(\mathbf{A}, \mathbf{b})$ which is determined by a \emph{basis} $\mathcal{I}^*\subset [K+L]$. 

\begin{definition}[BFS]
    Let $\mathbf{A}$ be an $M\times R$ matrix and $\mathbf{b}\in\mathbb{R}^M$. A subset $\mathcal{I}$ of $[R]$ is said to be a \emph{basis} of $\mathbf{A}$, if $|\mathcal{I}| = M$.
    A non-negative vector $\mathbf{x}^*\in \mathbb{R}_+^N$ is said to be a \emph{basic feasible solution} (BFS) of $(\mathbf{A}, \mathbf{b})$ corresponding to the \emph{basis} $\mathcal{I}$, if (i) $\mathbf{x}_i^* = 0$ for $i\not\in \mathcal{I}$; (ii) the square sub-matrix $\mathbf{A}_{\mathcal{I}}$ is invertible; and (iii) $\mathbf{x}_{\mathcal{I}}^* = \mathbf{A}_{\mathcal{I}}^{-1} \mathbf{b}\geq 0$. In this case, $\mathcal{I}$ is called a \emph{feasible basis} of $(\mathbf{A}, \mathbf{b})$.
\end{definition}

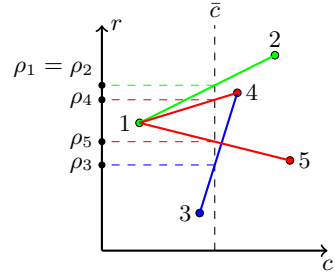
\begin{wrapfigure}{r}{0.4\textwidth}
    \centering\small
    \begin{tikzpicture}
        \coordinate (CBOUND) at (1, 0);
        \coordinate (CARROW) at (2.5, 0);
        \coordinate (RARROW) at (-0.5, 3);
        \draw[->,thick] (RARROW |- CARROW) -- (CARROW) node[anchor=north] {$c$};
        \draw[->,thick] (RARROW |- CARROW) -- (RARROW) node[anchor=west] {$r$};
        \draw[dashed,name path=vrt] (CBOUND) -- (CBOUND |- RARROW) node[anchor=south] {$\costbound$};
        \coordinate (A) at (0.0, 1.7);
        \coordinate (B) at (1.8, 2.6);
        \coordinate (C) at (0.8, 0.5);
        \coordinate (D) at (1.3, 2.1);
        \coordinate (E) at (2.0, 1.2);
        \draw[fill=green] (A) circle [radius=0.05] node[anchor=east] {$1$};
        \draw[fill=green] (B) circle [radius=0.05] node[anchor=south] {$2$};
        \draw[fill=blue] (C) circle [radius=0.05] node[anchor=east] {$3$};
        \draw[fill=red] (D) circle [radius=0.05] node[anchor=west] {$4$};
        \draw[fill=red] (E) circle [radius=0.05] node[anchor=west] {$5$};
        
        \draw[thick,green,name path=AB] (A) -- (B);
        \path[name intersections={of=vrt and AB,by=SCORE12}];
        \draw[dashed, green] (CBOUND |- SCORE12) -- (RARROW |- SCORE12);
        \draw[fill=black] (RARROW |- SCORE12) circle [radius=0.04] node[anchor=south east] {$\rho_1=\rho_2$};

        \draw[thick,blue,name path=CD] (C) -- (D);
        \path[name intersections={of=vrt and CD,by=SCORE3}];
        \draw[dashed, blue] (CBOUND |- SCORE3) -- (RARROW |- SCORE3);
        \draw[fill=black] (RARROW |- SCORE3) circle [radius=0.04] node[anchor=east] {$\rho_3$};

        \draw[thick,red,name path=AD] (A) -- (D);
        \path[name intersections={of=vrt and AD,by=SCORE4}];
        \draw[dashed, red] (CBOUND |- SCORE4) -- (RARROW |- SCORE4);
        \draw[fill=black] (RARROW |- SCORE4) circle [radius=0.04] node[anchor=east] {$\rho_4$};

        \draw[thick,red,name path=AE] (A) -- (E);
        \path[name intersections={of=vrt and AE,by=SCORE5}];
        \draw[dashed, red] (CBOUND |- SCORE5) -- (RARROW |- SCORE5);
        \draw[fill=black] (RARROW |- SCORE5) circle [radius=0.04] node[anchor=east] {$\rho_5$};
    \end{tikzpicture}
    \caption{Intersection value scores for a 5-arm, 1-constraint instance}
    \label{fig:ivscore1constr}
\end{wrapfigure}
For the ease of describing the algorithm, we will refer to the slack variable $s_l$ corresponding to the $l$-th constraint as ``arm $K+l$.'' With this convention, arms $K+1$ to $K+L$ can be thought of as ``virtual arms'' corresponding to slack variables. The proposed algorithm starts with the set of all arms $[K+L]$ and successively rejects one arm in each round. The algorithm returns when either it believes the problem is infeasible, or when there are only $L+1$ arms remaining. 

The choice of which arm to eliminate at each round is determined by a \textit{score function} (like an index in various bandit learning algorithms) for each arm computed from the empirical estimates of their rewards and costs: Let $\armset\subset [K+L]$ to the subset of remaining arms, Let $\hat{\mu}_{\armset}\in \mathbb{R}^{\armset}$ be the empirical mean vector defined by
\begin{equation}\label{eq:muhat}
    \hat{\mu}_{\armset} = (\hat{\mu}_a)_{a\in\armset},\qquad \hat{\mu}_a = \begin{cases}
        \hat{r}_a & a \leq K_0\\
        \mu_a & \text{otherwise}
    \end{cases}.
\end{equation}

Let the empirical constraint sub-matrix $\hat{\mathbf{A}}_{\armset}\in\mathbb{R}^{[L+1]\times \armset} $ be defined by
\begin{equation}\label{eq:Ahat}
    \hat{\mathbf{A}}_{\armset} = (\hat{\mathbf{A}}_{l, a})_{l\in [L+1], a\in \armset},\qquad \hat{\mathbf{A}}_{l, a} = \begin{cases}
        \hat{c}_{l, a}&l \leq L, a\leq K_0\\
        \mathbf{A}_{l, a}&\text{otherwise}
    \end{cases}.
\end{equation}

The \emph{intersection value (IV) score function} for arm $a\in \armset$  is then defined as
\begin{align}
    f_a^{\mathrm{IV}}(\hat{\mathbf{r}}, \hat{\mathbf{c}}) = \max\{\hat{\mu}^T_{\armset} \tilde{\mathbf{x}}: \tilde{\mathbf{x}} \text{ is a BFS of } (\hat{\mathbf{A}}_{\armset}, \mathbf{b}) \text{ corresponding to some basis $\mathcal{J}$ containing}~a\}\label{eq:IVscore}
\end{align}
with the convention that the maximum over an empty set is $-\infty$. 
In Figure \ref{fig:ivscore1constr}, we provide a visual presentation of the intersection value scores for the case $L=1$. The example also explains the name \emph{intersection value}, as it represents the intersections of line segments (for $L>1$, simplices) with the cost boundary (for $L>1$, faces of the constraint polytope).  

%\RJ{wrap text around it.}

The \fcnm{} algorithm will only pull unknown arms $a\in [K_0]$.
The number of times $T_k$ to pull each remaining unknown arm in round $k$ is defined as follows: Set 
$n_0 := 0$, and for $k\in [K-1]$, 
    \begin{equation}\label{eq:Tk}
        n_k := \left\lceil \dfrac{1}{\Psi(K_0, L)} \dfrac{N-K_0}{K+1-k} \right\rceil, ~~T_k := n_k - n_{k-1}, ~\text{where}~ \Psi(K_0, L) := \sum_{j=1}^{K_0} \frac{1}{\max(2, j-L)}. 
    \end{equation}
The \fcnm{} algorithm is formally presented in Algorithm \ref{algo:ivsr}.

\begin{algorithm}[!ht]
   \KwIn{Means of reward and costs of known arms $(r_a)_{a\in (K_0, K]}, (c_{l, a})_{a\in (K_0, K], l\in [L]}$; Cost bound $(\costbound_l)_{l\in [L]}$; Pulling budget $N$}
   \KwOut{Support of the best mixed arm and slack variables (or the symbol $\varnothing$ representing infeasibility) }
   \BlankLine
  Compute $(T_k)_{k=1}^{K-1}$ with \eqref{eq:Tk}\;
  Set $\armset = [K + L]$\;
	 \For{$k = 1$ to $K-1$} {
        % Compute $T_k =$ \RS{$N_{\text{remain}}, ~K,~L,~ |\armset\cap (K_0, K + L]|,~ k$}\;
	    Pull each arm $a\in\armset \cap [K_0]$ for $T_k$ times\;
        Update empirical means of reward $\hat{r}_a$ and costs $(\hat{c}_{l, a})_{l=1}^L$ for arms in $\armset \cap [K_0]$\;
        Compute the score $\hat{\rho}_a^k = f_a^{\mathrm{IV}}(\hat{\mathbf{r}}, \hat{\mathbf{c}})$ for each $a\in\armset$ through \eqref{eq:muhat}\eqref{eq:Ahat}\eqref{eq:IVscore} \;
        \If{$\max_{a\in\armset} \hat{\rho}_a^k = -\infty$}{
            \KwRet $\varnothing$\;
        }
        Eliminate the arm with the lowest score from $\armset$ (with arbitrary tie-breaking)\;
	}
    \KwRet $\armset$\;
    \caption{\Fancyname ~(\fcnm)}
    \label{algo:ivsr}
\end{algorithm}

\begin{remark}
    To use Algorithm \ref{algo:ivsr} to estimate the best mixed arm (i.e., the support \emph{and} the associated probabilities), one can simply construct the empirical constraint sub-matrix $\hat{\mathbf{A}}_{\armset}$ according to \eqref{eq:Ahat} and compute $\hat{x}_{\armset}^* = \hat{\mathbf{A}}_{\armset}^{-1} \mathbf{b}$. Note that $\hat{x}_{\armset}^*$ may include slack variables as well. 
    %Maybe we can say that $\hat{x}_{\armset \cap [K_0] }^*$ is the estimated best mixed arm.
\end{remark}

Note that at the beginning of round $k$, the set $\armset$ of remaining arms contains some true arms and may contain some virtual arms (corresponding to the slack variables). Whether a true or a virtual arm will be eliminated in round $k$ depends on the random realizations of the rewards and costs. Thus, unlike the classical Successive Reject algorithm \cite{audibert2010best}, the number of true arms remaining after each round in our algorithm is a random variable. 
While the total number of (true) arm pulls by our algorithm is random,  we can show that we will always meet the pulling budget $N$. 
\begin{proposition}\label{prop:budget}
    Under Algorithm \ref{algo:ivsr}, the number of total arm pulls never exceeds $N$. 
\end{proposition}
We relegate the proof to Appendix \ref{app:proofs}.

\paragraph{Using a different score function.} \label{sec:ivsr2}
Instead of the intersection value score function $f^{\mathrm{IV}}$, we can also use another function $f^{\mathrm{L}}$, called the Lagrangian function, that comes from linear programming duality theory. The Lagrangian score function is defined as follows: Let $\hat{\mu}_{\armset}$ and $\hat{\mathbf{A}}_{\armset}$ follow the definitions in \eqref{eq:muhat}\eqref{eq:Ahat}. Let $\hat{\lambda}^*\in\mathbb{R}^{L+1}$ be an optimal solution to the empirical dual linear program $\min_{\lambda\in \mathbb{R}^{L+1}}\{\mathbf{b}^T \lambda: \hat{\mathbf{A}}_{\armset}^T \lambda \geq \hat{\mu}_{\armset} \}$. We define
\begin{align}
    f_a^{\mathrm{L}}(\hat{\mathbf{r}}, \hat{\mathbf{c}}) = 
    \begin{cases}
        \left(\hat{\mu}_{\armset} - \hat{\mathbf{A}}_{\armset}^T \hat{\lambda}^*\right)_{a}&\text{empirical dual LP is bounded}\\
        -\infty &\text{otherwise}
    \end{cases}
\end{align}
 This yields another flavor of the \fcnm{} algorithm that we call \texttt{\fcnmL{}}. 
 In Appendix \ref{app:moreexp}, we will see that the \texttt{\fcnmL{}} algorithm on some problem instances can perform better than the \fcnm{} algorithm.

\section{Analysis}
\subsection{Preliminaries}\label{sec:gap}
We first introduce the following mild assumption that we will use for our analysis.

\begin{assumption}\label{assump:unique}
    The linear program \eqref{eq:SFLP} has a unique optimal solution $\mathbf{x}^*$ with exactly $L+1$ non-zero coordinates. 
\end{assumption}

\begin{remark}
    Assumption \ref{assump:unique} does not restrict the best mixed arm to be a strict mix of $L+1$ arms: Note that $\mathbf{x}^*$ contains both the mixing probabilities and the slack variables for the constraints. Assumption \ref{assump:unique} requires that if the optimal mixed arm is a mix of $m$ arms, then there need to be exactly $L+1-m$ non-binding constraints under this mixed arm. 
\end{remark}

The uniqueness of optimal solution is a standard assumption in best arm identification problems (e.g. \cite{audibert2010best,kaufmann2016complexity,faizal2022constrained}). 
The further assumption on the size of the support is necessary for CBMAI problems since it ensures the stability of the optimal solution: Without this assumption, an infinitesimal change of the cost matrix $\mathbf{c}$ could result in a change of the support of the best mixed arm. In this case, identifying the support of best mixed arm beyond a certain probability would be impossible under any budget $N$, since it requires estimating $\mathbf{c}$ with infinite precision.

Next, given an instance satisfying Assumption \ref{assump:unique} we formally define the gaps $\Delta_0, (\Delta_{(i)})_{i\in [K+L]}$ that characterize the hardness of the instance. These gaps appear in both the upper bound on the error probability, (Theorem \ref{thm:main}) as well as on its lower bound (Theorem \ref{prop:lowerbound}). 

Let $\mathcal{I}^* = \mathrm{supp}(\mathbf{x}^*)$ denote the support of the optimal solution (or the \emph{optimal basis}).
For each basis set $\mathcal{J} \subset [K+L], |\mathcal{J}| = L + 1$, define the \emph{basis value gap} of $\mathcal{J}$ by
\begin{equation}\label{eq:deltaJ}
\begin{split}
    \TD_{\mathcal{J}}^2 
    &= \inf_{\Tilde{\mathbf{r}}\in\mathbb{R}^{K_0}, \Tilde{\mathbf{c}}\in\mathbb{R}^{L\times K_0}}\left\{\sigma_r^{-2} \sum_{a = 1}^{K_0} (r_a-\Tilde{r}_a)^2 + \sigma_c^{-2} \sum_{l=1}^L \sum_{a = 1}^{K_0} (c_{l, a}-\Tilde{c}_{l, a})^2 : 
    \right.\\&\hspace{8em}\left.
    \Tilde{\mathbf{A}}_{\mathcal{I}^*}^{-1} \mathbf{b}\geq 0, \Tilde{\mathbf{A}}_{\mathcal{J}}^{-1} \mathbf{b}\geq 0,~ \Tilde{\mu}_{\mathcal{J}}^T \Tilde{\mathbf{A}}_{\mathcal{J}}^{-1} \mathbf{b} \geq \Tilde{\mu}_{\mathcal{I}^*}^T \Tilde{\mathbf{A}}_{\mathcal{I}^*}^{-1} \mathbf{b} \right\}
\end{split}
\end{equation}
where $\Tilde{\mathbf{A}}, \Tilde{\mu}$ are defined through $(\Tilde{r}, \Tilde{c})$ in the same way as how $\hat{\mathbf{A}}, \hat{\mu}$ are defined through $(\hat{\mathbf{r}}, \hat{\mathbf{c}})$ in \eqref{eq:muhat}\eqref{eq:Ahat}. 
We follow the convention that the infimum of an empty set is $+\infty$. The basis value gap represents the minimum distance one needs to move $(\mathbf{r}, \mathbf{c})$ to an alternative instance $(\tilde{\mathbf{r}}, \tilde{\mathbf{c}})$ where the expected reward under the originally optimal basis $\mathcal{I}^*$ is overtaken by that of another basis $\mathcal{J}$ while preserving the feasibility of $\mathcal{I}^*$.

Note that in \eqref{eq:deltaJ}, the infimum can be attained by moving only the rewards and costs associated with arms in $(\mathcal{J}\cup\mathcal{I}^*)\cap[K_0]$. We write \eqref{eq:deltaJ} as an infimum over all reward-and-cost vectors for the sake of consistency and ease of notations. 

\begin{proposition}\label{prop:Deltanonzero}
    Under Assumption \ref{assump:unique}, $\mathcal{J}\neq \mathcal{I}^*$ if and only if $\Delta_{\mathcal{J}}^2 > 0$.
\end{proposition}

We relegate the proof to Appendix \ref{app:prop:Deltanonzero}.

Furthermore, for each $a\in [K+L]$, we define the \emph{arm value gap} of $a$ as
\begin{equation}\label{eq:Deltaa}
    \TD_a^2 = \min\{\TD_{\mathcal{J}}^2: a\in \mathcal{J}\subset [K+L],~ |\mathcal{J}| = L+1\}.    
\end{equation}

Following \cite{audibert2010best}, let $(k)$ denote the arm (including virtual arms) with the $k$-th smallest $\TD_a$ among all arms $a\in [K+L]$. Under Assumption \ref{assump:unique}, it follows from Proposition \ref{prop:Deltanonzero} that $0 = \TD_{(1)} = \cdots = \TD_{(L+1)} < \TD_{(L+2)} \leq \cdots \leq \TD_{(K+L)}$. 

In addition to the above, define the \emph{optimal support infeasibility gap} as
\begin{align}
    \Delta_0^2 = \sigma_c^{-2} \inf_{\Tilde{c}\in\mathbb{R}^{L\times K_0} }\left\{ \sum_{l=1}^L\sum_{a=1}^{K_0}(c_{l, a} - \Tilde{c}_{l, a})^2: \det(\Tilde{\mathbf{A}}_{\mathcal{I}^*}) = 0 \text{ or }\Tilde{\mathbf{A}}_{\mathcal{I}^*}^{-1} \mathbf{b} \not\geq 0 \right\}\label{eq:Delta0}
\end{align}

The fact that $\Delta_0^2 > 0$ under Assumption \ref{assump:unique} can be established via the continuity and strict positiveness of the mappings $\Tilde{c}\mapsto \det(\Tilde{\mathbf{A}}_{\mathcal{I}^*})$ and $\Tilde{c}\mapsto \Tilde{\mathbf{A}}_{\mathcal{I}^*}^{-1} \mathbf{b}$ at $\Tilde{c} = \mathbf{c}_{[K_0]}$. 

\subsection{Upper Bound on the Error Probability of the \fcnm{} Algorithm}
\label{sec:ivsr-ub}
We now provide an upper bound on the mis-identification probability of the \fcnm{} algorithm under Assumption \ref{assump:unique}.  

\begin{theorem}\label{thm:main}
   Let $\armset^{\fcnm}$ denote the output of Algorithm \ref{algo:ivsr}. Then, under Assumption \ref{assump:unique}, we have
    \begin{equation}\label{eq:main}
        \Pr(\armset^{\fcnm} \neq \mathcal{I}^* ) \leq \bigO_L(K)\cdot \exp\left(-\dfrac{\Tilde{N} \Delta_0^{2}}{ K }\right) + \bigO_L(K^{L+2}) \cdot\exp\left(-\min_{2\leq i\leq K}\dfrac{\Tilde{N}\Delta_{(i+L)}^2 }{i} \right)
    \end{equation}
    where $\Tilde{N} := \frac{N - K_0}{3(\frac{L+1}{2} + \log K_0)}$ and $\bigO_L$ is the standard big-O notation where $L$ is treated as a constant. %and $\Delta_0 > 0$ and $0 < \Delta_{(L+2)} \leq \Delta_{(L+3)} \leq \cdots \Delta_{(K+L)}$ are instance-dependent constants defined in Appendix \ref{sec:gap}. 
\end{theorem}

\begin{proof}[Proof Sketch]
    The proof follows some steps similar to those in \cite{audibert2010best}. We first show that if the algorithm fails to output $\mathcal{I}^*$, then it is necessary that at some round $k$, either (i) the basis $\mathcal{I}^*$ becomes infeasible under the empirical mean costs $\hat{\mathbf{c}}$; or (ii) the basis $\mathcal{I}^*$ remains feasible, but its corresponding expected reward is overtaken by some other basis $\mathcal{J}$ under the empirical rewards $\hat{\mathbf{r}}$ and costs $\hat{\mathbf{c}}$. For either of the above events to happen, it is necessary that the empirical mean reward and costs $(\hat{\mathbf{r}}, \hat{\mathbf{c}})$ deviate from the true means $(\mathbf{r}, \mathbf{c})$ by a certain distance characterized by gaps defined in Section \ref{sec:gap}. Then, the probability of either event can be bounded with concentration inequalities. We then conclude the result with union bound.
\end{proof}
    See Appendix \ref{app:thm:main} for the detailed proof.

\begin{remark} 
Theorem \ref{thm:main} shows that the error is dominated by two components: The quantity $\Delta_0$ describes how close the optimal mixed arm is to \emph{infeasibility}, while $\TD_{(i+L)}$ describes how close the optimal mixed arm is to \emph{other candidates} for the best mixed arm. This echoes the result of \cite{faizal2022constrained} for constrained best deterministic arm identification problems. 
\end{remark}

\subsection{Lower Bound}\label{sec:LB}
As we stated above, our objective is to design a pure-exploration algorithm $\phi$ that will minimize the mis-identification probability. But how would we know that our upper bound on it is tight, or not? To characterize that, we introduce a lower bound with which the upper bound can then be compared. 

To meaningfully derive a lower bound on the performance of any CBMAI algorithm, it is essential to specify the class of instances to be considered. (It's not so difficult to design algorithms that achieve uniformly good performance on a small class of instances, since the algorithm only needs to distinguish between them.) We consider a class of instances with Gaussian rewards and costs such that the variances $\sigma_c^2, \sigma_r^2$ are fixed and known, so that an instance is parameterized by $\theta = (\theta_a)_{a\in [K]} = (r_{a}, c_{1, a}, \cdots, c_{L, a})_{a\in [K]}$. For simplicity, we consider $K_0 = K$, i.e., all arms are unknown.
We define $\varTheta$ to be a class of instances $\theta$ that either (i) has no feasible solution, or (ii) satisfies Assumption \ref{assump:unique}. For $\theta\in\varTheta$, define $\mathcal{I}^{\theta} = \varnothing$ if $\theta$ is an instance with no feasible solution. Otherwise, define $\mathcal{I}^{\theta}$ to be the optimal basis for this instance.

We consider a class of algorithms that satisfy the following \emph{consistency} requirement.

\begin{definition}[Consistency \citep{barrier2023best}]
    Let $\phi_N$ be an algorithm for CBMAI with budget $N$. A sequence of algorithms $(\phi_N)_{N=N_0}^\infty$ is said to be \emph{consistent} if for any instance $\theta\in\varTheta$, $\Pr_{\theta}(\mathcal{X}^{\phi_N} \neq \mathcal{I}^{\theta}) \rightarrow 0$ as $N\rightarrow+\infty$.
\end{definition}

The consistency condition means that given sufficient amount of budget, an algorithm can eventually: (i) identify the optimal basis when it is possible to do so, and (ii) output $\varnothing$ whenever the instance has no feasible solution. 

It can be shown that the naive Uniform Sampling and Linear Program (\texttt{USLP}) algorithm (i.e., pull each arm $\lfloor N/K\rfloor$ times, compute the empirical means of rewards and costs of all arms, and then solve the empirical version of \eqref{eq:SFLP}) is a consistent algorithm. We note that the \fcnm{} algorithm is a consistent algorithm. 

\begin{theorem}\label{prop:lowerbound}
    For any consistent algorithm $\phi_N$, under any instance $\theta$ satisfying Assumption \ref{assump:unique}, the mis-identification probability satisfies
    \begin{equation}
        \limsup_{N\rightarrow\infty} -\dfrac{1}{N} \log\Pr_{\theta}(\mathcal{X}^{\phi_N} \neq \mathcal{I}^{\theta}) \leq \dfrac{1}{2}\min\{\Delta_0^2, \Delta_{(L+2)}^2\}.
    \label{eq:lb}
    \end{equation}
    %where $\Delta_0, \Delta_{(L+2)}$ are instance-dependent constants defined in Appendix \ref{sec:gap}.

    Furthermore, for the \texttt{USLP} algorithm, we have
    \begin{equation}
        \limsup_{N\rightarrow\infty} -\dfrac{1}{N} \log\Pr_{\theta}(\mathcal{X}^{\mathrm{USLP}_N} \neq \mathcal{I}^{\theta}) \leq \dfrac{1}{2K}\min\{\Delta_0^2, \Delta_{(L+2)}^2\}
    \label{eq:ub-uslp}
    \end{equation}
\end{theorem}
The proof can be found in Appendix \ref{app:lowerbound}, and uses techniques inspired by \cite{kaufmann2016complexity} and \cite{barrier2023best}.

For Algorithm \ref{algo:ivsr},  Theorem \ref{thm:main} yields that
\begin{equation}
\begin{split}&\quad~\liminf_{N\rightarrow\infty} -\dfrac{1}{N} \log\Pr_{\theta}(\mathcal{X}^{\mathrm{\fcnm}_N} \neq \mathcal{I}^{\theta}) \\
    &\geq \dfrac{1}{3\left(\frac{L+1}{2}+\log K\right)} \min\left\{ \frac{\Delta_0^2}{K}, \frac{\Delta_{(L+2)}^2}{2}, \frac{\Delta_{(L+3)}^2}{3}, \cdots, \frac{\Delta_{(L+K)}^2}{K} \right\}.
\end{split}
\label{eq:ub-sfsr}
\end{equation}
We can now compare the upper bound for the \fcnm ~algorithm above to the lower bound in \eqref{eq:lb} and observe the presence of several common terms. Note that \textit{tight} instance-dependent lower bounds for fixed budget identification problems are not known even for unconstrained BAI problems \citep{qin2022open}. So, while the lower bound we provide in  Theorem \ref{prop:lowerbound} may not be tight, it does show that the gaps $\TD_0$ and $(\TD_{(L+i)})_{i=2}^K$ are appropriate indicators of the hardness of a CBMAI problem: If one of them is very small, then the CBMAI problem is difficult for any consistent algorithm to handle. 
Instance independent min-max lower bound of the type in \cite{carpentier2016tight} can also be derived but are not very meaningful for this problem since one can always construct hard CBMAI instances by translating hard unconstrained BAI problems. 

% \begin{remark}
%     We choose not to provide the instance independent min-max lower bound of the type in \cite{carpentier2016tight} (i.e. given a bound $h$ on the complexity measure $H_2$, one can construct some hard instances such that no algorithm can have error bound smaller than some function of $N/h$.): Such an result is not very interesting for our setting, since one can always construct hard CBMAI instances through translating hard unconstrained BAI problems (for example, by setting cost bounds $(\costbound_l)_{l=1}^L$ to be sufficiently large and setting $\sigma_c$ to be close to $0$). Such a result, if formulated, does not provide any new insight to CBMAI problems beyond what we already know about BAI.
% \end{remark}

\section{Empirical Performance}\label{sec:empirical}
In this section, we compare the empirical performance of the two flavors of \fcnm{} with the naive \texttt{USLP} algorithm. We only presented instances with $L=1$ here. Empirical results for instances with more than one constraints are included in Appendix \ref{app:moreexp}.

\begin{figure}[!ht]
    \centering
    \includegraphics[width=0.75\textwidth]{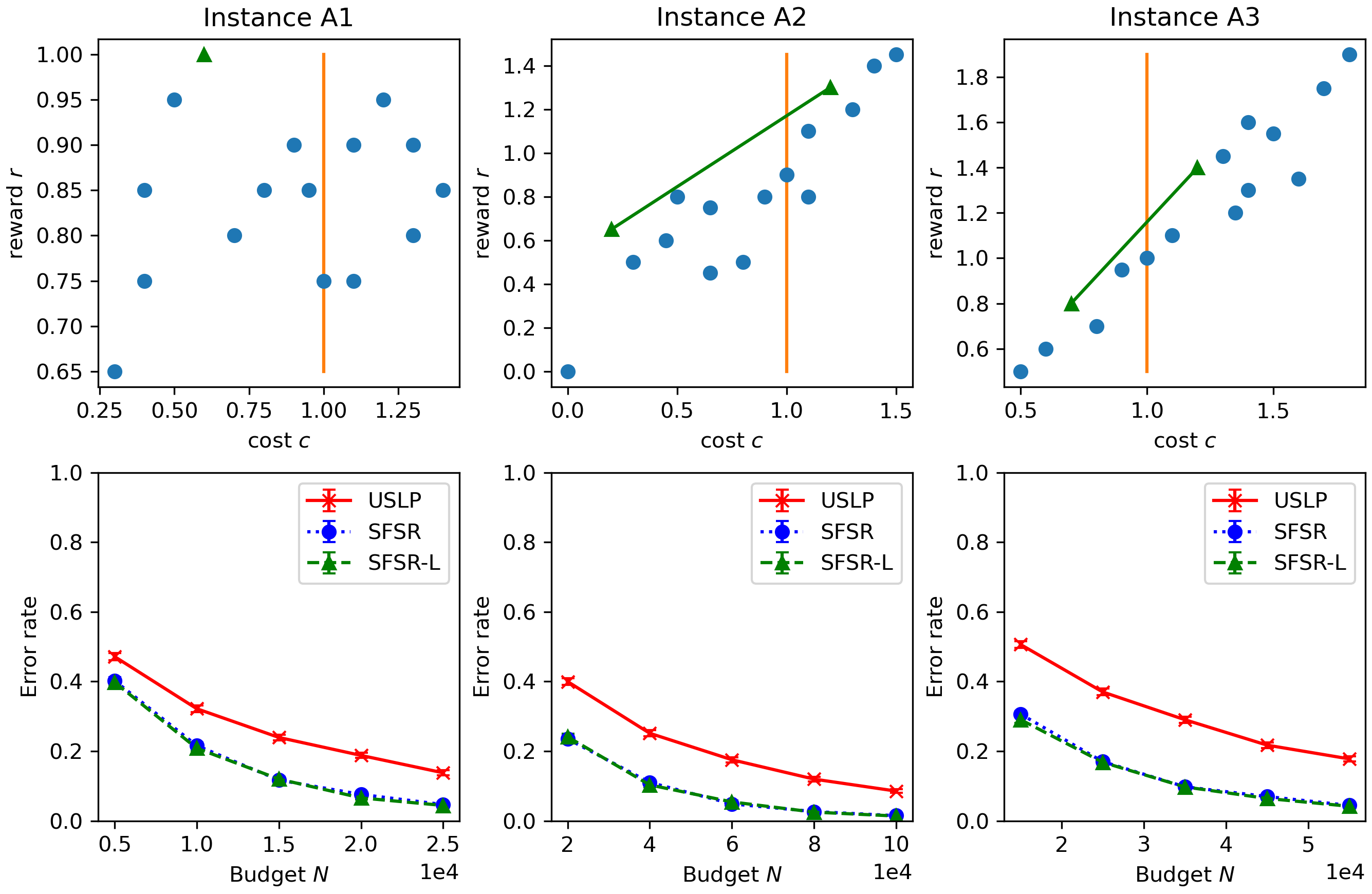}
    \caption{Top: Three 16-arm instances. Arms in the optimal support are labeled with green triangles. Bottom: Empirical results for three algorithms under varying budgets. 95\% confidence intervals are indicated and tight. }
    \label{fig:expres:rand}
\end{figure}

\begin{figure}[!ht]
    \centering
    \includegraphics[width=0.75\textwidth]{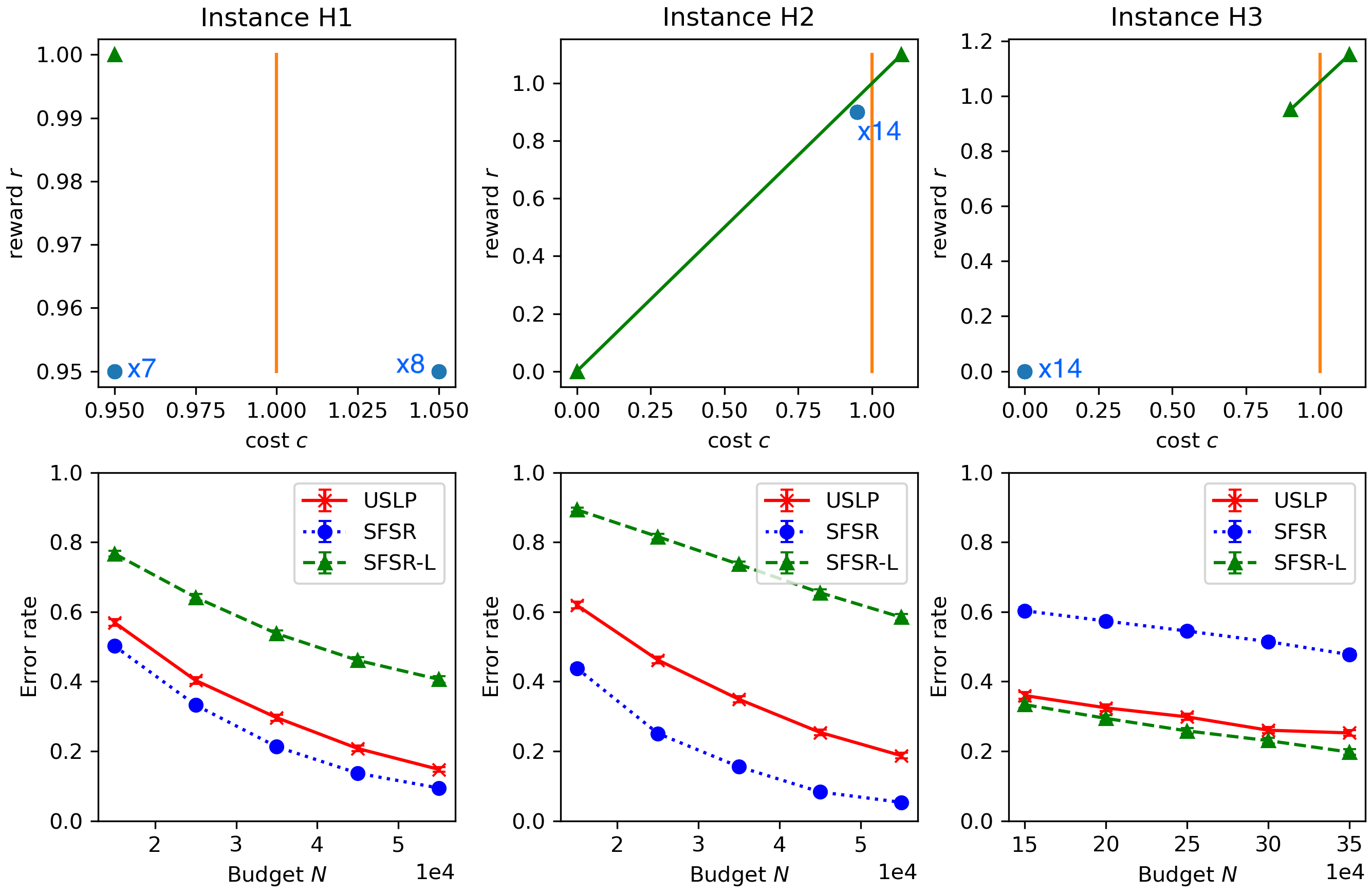}
    \caption{Top: Three hard 16-arm instances. Arms in the optimal support are labeled with green triangles. Bottom: Empirical results for three algorithms under varying budgets. 95\% confidence intervals are indicated and tight.}
    \label{fig:expres:path}
\end{figure}
In all of the experiments, we set $K_0=K=16, \sigma_r = 1, \sigma_c = 0.5,$ and $\Bar{c}_1 = 1$. For each combination of instance-algorithm-budget, we conduct 10,000 independent runs to obtain the error rate as the proportion of times the algorithm produces the wrong support. In every figure, we added (tiny) error bars to represent a $95\%$ confidence interval for the error rate. We implemented the experiments in Python and conducted the experiments on an Apple M1 MacBook Air. Each figure takes about 15 minutes to generate.

We first consider three arbitrary instances $A1, A2$ and $A3$ in Figure \ref{fig:expres:rand}, where certain arms are clearly sub-optimal while others are not. The instances and their corresponding results are shown in Figure \ref{fig:expres:rand}. This includes one instance where the optimal arm is deterministic, and two instances where the optimal arm is a strict mix of two arms. The baseline is the Uniform Sampling and Linear Program (\texttt{USLP}) algorithm that pulls each arm $\lfloor N/K\rfloor$ times, computes the empirical means of rewards and costs of all arms, and then solves  the empirical version of \eqref{eq:SFLP}) (there is no other known algorithm in the literature or otherwise). The results show that both \fcnm{} and \fcnmL{} clearly outperform \texttt{USLP}, and the two flavors of \fcnm{} have nearly no discernible difference in performance. Furthermore, we also observe that the error rate decreases exponentially in $N$.

While both flavors of \fcnm{} can achieve good performance on an average instance, in Figure \ref{fig:expres:path}, we see that in certain carefully constructed hard instances ($H1$-$H3$) while one of the two flavors, \fcnm{} or \fcnmL{} performs well, the other does not (the guarantee in any case is probabilistic). In fact, the $H3$ instance in Figure \ref{fig:expres:path} shows that it is hard enough that \fcnmL{} struggles to perform much better than the \texttt{USLP} algorithm. In Appendix \ref{app:pdsrfail}, we provide a detailed explanation on why \fcnmL{} fails on instance $H2$.

\section{Conclusions}
\label{sec:conclusions}

In this paper, we introduced the constrained best mixed arm identification (CBMAI) problem. While prior work has considered such a constrained problem with deterministic arms, it is well known that one can do better allowing for mixed arms. Unfortunately, the mixed arm problem is much more challenging due to there being uncountably many of them. We have proposed the first algorithm for the CBMAI problem that we are able to show theoretically and empirically also has very good performance in terms of the error probability decreasing exponentially in $N$. The problem is of wide interest in many practical settings that often have multiple objectives but with unknown reward and cost models. For example, the cost attributes can even be related to AI safety (e.g., do not recommend offensive videos, generate undesirable images, etc.). 

Our work provides a basis for further extensions in a number of directions: One could consider contextual bandit models, e.g., linear bandit models that have wide applicability in recommendation systems. A fixed confidence version of CBMAI problem is also interesting, which has only been solved under the known constraints case \cite{carlsson2023pure}. This work could also be extended to an Constrained MDP setting to find the best policy that also obeys average constraints. 
There is  probably also some scope to design an even better algorithm by combining various score functions. 

We hope this paper will inspire more work in this direction that is important for so many applications.

\newpage

\printbibliography

\newpage

\appendix

\section{Example of failure of the \fcnmL{} Algorithm}\label{app:pdsrfail}

\begin{figure}[!ht]
    \centering\small
    \begin{tikzpicture}
        \coordinate (CBOUND) at (1, -0.5);
        \coordinate (CARROW) at (2.5, -0.5);
        \coordinate (RARROW) at (-0.5, 2.0);
        \draw[->,thick] (RARROW |- CARROW) -- (CARROW) node[anchor=north] {$c$};
        \draw[->,thick] (RARROW |- CARROW) -- (RARROW) node[anchor=west] {$r$};
        \draw[dashed,name path=vrt] (CBOUND) -- (CBOUND |- RARROW) node[anchor=south] {$\costbound$};
        \coordinate (A) at (-0.2, -0.2);
        \coordinate (B) at (1.2, 1.2);
        \coordinate (C) at (0.9, 0.75);
        % \coordinate (D) at (1.2, 0.9);
        \draw[thick,green,name path=AB] (-0.5, -0.5) -- (B);
        \foreach \x in {C}
        \draw[red] (\x) -- (intersection of \x--\x|-RARROW and A--B);
        \draw[fill=green] (A) circle [radius=0.05] node[anchor=south] {$1$};
        \draw[fill=green] (B) circle [radius=0.05] node[anchor=south] {$2$};
        \draw[fill=blue] (C) circle [radius=0.05] node[anchor=north west] {};
        \draw ($(C) + (0.2, -0.2)$) node[anchor=north west] {$3,\cdots, K$};
        \draw[->] ($(C) + (0.3, -0.3)$) -- ($(C) + (0.1, -0.1)$);
        % \draw[fill=red] (D) circle [radius=0.05] node[anchor=west] {$3$};
    \end{tikzpicture}
    \qquad
    \begin{tikzpicture}
        \coordinate (CBOUND) at (1, -0.5);
        \coordinate (CARROW) at (2.5, -0.5);
        \coordinate (RARROW) at (-0.5, 2.0);
        \draw[->,thick, name path=cax] (RARROW |- CARROW) -- (CARROW) node[anchor=north] {$\hat{c}$};
        \draw[->,thick, name path=rax] (RARROW |- CARROW) -- (RARROW) node[anchor=west] {$\hat{r}$};
        \draw[dashed,name path=vrt] (CBOUND) -- (CBOUND |- RARROW) node[anchor=south] {$\costbound$};
        \coordinate (A) at (-0.2, -0.2);
        \coordinate (B) at (1.25, 1.10);
        \coordinate (C) at (0.9, 0.75);
        % \coordinate (D) at (1.2, 0.85);
        \coordinate (C1) at ($(C) + (-0.08, 0.105)$);
        % \coordinate (C1E) at (-0.5, 0.5166666);
        \coordinate (C1E) at (intersection of C1--B and RARROW|- CARROW--RARROW);
        \coordinate (C2) at ($(C) + (0.05, -0.1)$);
        \coordinate (C3) at ($(C) + (-0.1, -0.02)$);
        \coordinate (C4) at ($(C) + (0.1, 0.1)$);
        \draw[thick,green,name path=BC1] (B) -- (C1E);
        \foreach \x in {A, C, C2, C3, C4}
        \draw[red] (\x) -- (intersection of \x--\x|-RARROW and B--C1E);
        \draw[fill=green] (A) circle [radius=0.05] node[anchor=west] {$1$};
        \draw[fill=green] (B) circle [radius=0.05] node[anchor=south] {$2$};
        \draw[fill=blue] (C) circle [radius=0.05] node[anchor=north west] {};
        \draw[fill=blue] (C2) circle [radius=0.05] node[anchor=north west] {};
        \draw[fill=blue] (C3) circle [radius=0.05] node[anchor=north west] {};
        \draw[fill=blue] (C4) circle [radius=0.05] node[anchor=north west] {};
        \draw[fill=blue] (C1) circle [radius=0.05] node[anchor=south east] {};
        \draw ($(C1) + (-0.1, 0.1)$) node[anchor=south east] {$j$};
        \draw[->] ($(C1) + (-0.2, 0.2)$) -- ($(C1) + (-0.08, 0.08)$);
        % \draw[fill=red] (D) circle [radius=0.05] node[anchor=west] {$3$};
    \end{tikzpicture}
    \caption{Illustration of \fcnmL{} in a 1-constraint instance. Left: True mean reward and cost. Right: Empirical means after the first episode. Despite that the empirical means do not deviate from the the true mean by too much, arm 1 (a member of the optimal support) ends up having the lowest empirical Lagrangian reward, and is eliminated as a result.}
    \label{fig:pdsrproblem}
\end{figure}
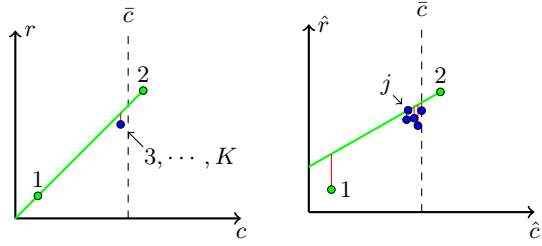

Consider a CBMAI instance with $K$ arms and one type of cost. The mean reward and cost are shown as in the left of Figure \ref{fig:pdsrproblem}: Arm 1 has low reward and low cost, arm 2 has high reward and near feasible cost, and arm 3 to $K$ all have the same mean reward and cost: The cost is feasible but close to cost bound $\costbound$, and the reward is chosen such that the best mixed arm is formed by a mixture of arm 1 and 2. 
The (negative) Lagrangian reward $f_a^{\mathrm{L}}(\hat{\mathbf{r}}, \hat{\mathbf{c}})$ of an arm $a\in [K]$ can be visualized in Figure \ref{fig:pdsrproblem} as the vertical distance between arm $a$ to the ``frontier'' (i.e. the extended line formed by cost-reward vectors of two arms in the empirical optimal support).

Now, consider the end of episode 1 of \fcnmL{}, and the empirical means of rewards and costs are shown as on the right of Figure \ref{fig:pdsrproblem}. Now, arm 2 and some arm $3\leq j\leq K$ forms the empirical frontier. The empirical dual optimal solution (symbolized by the slope of the frontier) is very different from the true dual optimizer. More importantly, the empirical Lagrangian reward of arm 1 is now the lowest among all arms, and arm 1 is rejected by the \fcnmL{} algorithm in round 1 as a result. Note the identification error happens despite the fact that the arm 1 did not underperform (i.e. $\hat{r}_1 < r_1, \hat{c}_1 > c_1$) its mean.

While in elimination style algorithms there's always the possibility of erroneously rejecting optimal arms, we note that the type of event as shown on the right of Figure \ref{fig:pdsrproblem} is not unlikely: We only require \emph{one of} the $K-2$ arms to slightly outperform its true mean for arm 1 to be eliminated. In comparison, in unconstrained BAI problems, for the optimal arm (arm 1) to be rejected in episode 1 in an elimination-style algorithm \citep{audibert2010best,karnin2013almost}, it requires \emph{all of} the other arms (including the worst arm) to empirically outperform arm 1.

\section{Proofs of Propositions and Theorems}
\label{app:proofs}

\subsection{Auxilliary Results}\label{app:aux}
\begin{lemma}\label{lem:chi2}
    Let $\chi_m^2$ be a chi-squared variable with degree $m$ and $t > 0$, then $\Pr(\chi_m^2 \geq t) \leq 3^{m/2}\exp(-t/3)$. 
\end{lemma}

\begin{proof}
    The moment generating function of $\chi_m^2$ is given by $\E[\exp(\zeta \chi_m^2)] = (1-2\zeta)^{-m/2}$ for $\zeta < 1/2$. Through Markov Inequality we have
    \begin{align}
        \Pr(\chi_m^2 \geq t) &\leq \E[\exp(\zeta \chi_m^2)] e^{-\zeta t} = (1-2\zeta)^{-m/2} e^{-\zeta t}
    \end{align}
    The proof is completed by picking $\zeta = 1/3$.
\end{proof}

\subsection{Proof of Proposition \ref{prop:budget}}\label{app:prop:budget}
Define $n_{k} = n_{K-1}$ for $k> K-1$. Imagine that in each episode, the algorithm pulls all arms still remaining in $\armset$ (including virtual arms and known arms). Then the total number of arm pulls is $\sum_{k=1}^{K+L} n_k$. 

If an ``arm'' $i$ is the $k$-th rejected arm, then it is pulled exactly $n_k$ times. If it is not rejected, then it is pulled $n_{K-1}$ times.
Hence to obtain the actual total number of arm pulls, we only need to subtract those $n_k$'s corresponding to virtual arms and known arms from the summation.
    We conclude that the total number of arm pulls is at most
    \begin{align}
        &\quad~ \max_{\substack{\mathcal{J}\subset [K+L]\\|\mathcal{J}| = K - K_0 + L}}\left( \sum_{k=1}^{K+L} n_k - \sum_{k\in \mathcal{J}} n_k\right) = \sum_{k=K-K_0 + L + 1}^{K+L} n_k\\
        &\leq \sum_{k=K-K_0 + L + 1}^{K+L}\left( 1 + \dfrac{1}{\Psi(K_0, L)} \dfrac{N-K_0}{\max(2, K + 1 - k)} \right)  \\
        &= K_0 + (N - K_0) \cdot \dfrac{1}{\Psi(K_0, L)} \sum_{j= L}^{K_0+L-1} \dfrac{1}{\max(2, K_0 - j)} = N
    \end{align}

\subsection{Proof of Proposition \ref{prop:Deltanonzero}}\label{app:prop:Deltanonzero}
    The ``if'' part is clear by the definition of $\mathcal{I}^*$. We establish the ``only if'' part as follows.

    Consider a basis $\mathcal{J}\neq \mathcal{I}^*$. Suppose that $\TD_{\mathcal{J}}^2 = 0$. Then, there exists a sequence of reward-cost vectors $(\Tilde{r}^{(n)}, \Tilde{c}^{(n)})_{n=1}^\infty$ such that both of the following hold: (i) $(\Tilde{r}^{(n)}, \Tilde{c}^{(n)}) \rightarrow (r, c)$; (ii) $(\Tilde{\mathbf{A}}_{\mathcal{I}^*}^{(n)})^{-1} \mathbf{b}\geq 0, (\Tilde{\mathbf{A}}_{\mathcal{J}}^{(n)})^{-1} \mathbf{b}\geq 0,~ (\Tilde{\mu}_{\mathcal{J}}^{(n)})^T (\Tilde{\mathbf{A}}_{\mathcal{J}}^{(n)} )^{-1} \mathbf{b} \geq (\Tilde{\mu}_{\mathcal{I}^*}^{(n)})^T (\Tilde{\mathbf{A}}_{\mathcal{I}^*}^{(n)})^{-1} \mathbf{b}$. 

    Set $\mathbf{x}^{(n)}\in \mathbb{R}^{K+L}$ to be the basic feasible solution of $(\Tilde{\mathbf{A}}^{(n)}, \mathbf{b})$ corresponding to basis $\mathcal{J}$, i.e. $\mathbf{x}_{\mathcal{J}}^{(n)}= (\Tilde{\mathbf{A}}_{\mathcal{J}}^{(n)})^{-1} \mathbf{b}$ and $\mathbf{x}_{i}^{(n)} = 0$ for $i\not\in \mathcal{J}$. Note that $(\mathbf{x}^{(n)})_{n=1}^\infty$ is a uniformly bounded sequence of finite dimensional vectors. By taking subsequences, without lost of generality, assume that $\mathbf{x}^{(n)}\rightarrow \mathbf{x}^{(\infty)}$. We have
    \begin{equation}
        \mathbf{x}^{(\infty)}\geq 0,\quad \mathbf{A}\mathbf{x}^{(\infty)} = \lim_{n\rightarrow\infty} (\Tilde{\mathbf{A}}_{\mathcal{J}}^{(n)}) \mathbf{x}_{_{\mathcal{J}}}^{(n)} = \mathbf{b},
    \end{equation}
    meaning that $\mathbf{x}^{(\infty)}$ is a feasible solution of \eqref{eq:SFLP}. By taking the limit of the last inequality in (ii) we have 
    \begin{equation}\label{eq:xinfinityoptimal}
        \mu^T \mathbf{x}^{(\infty)} = \lim_{n\rightarrow\infty} (\Tilde{\mu}_{\mathcal{J}}^{(n)})^T (\Tilde{\mathbf{A}}_{\mathcal{J}}^{(n)} )^{-1} \mathbf{b} \geq \limsup_{n\rightarrow\infty} (\Tilde{\mu}_{\mathcal{I}^*}^{(n)})^T (\Tilde{\mathbf{A}}_{\mathcal{I}^*}^{(n)})^{-1} \mathbf{b}.
    \end{equation}

    Under Assumption \ref{assump:unique}, $\mathbf{A}_{\mathcal{I}^*}$ is invertible and hence the mapping $\tilde{c}\rightarrow \Tilde{\mathbf{A}}_{\mathcal{I}^*}^{-1}$ is continuous at $\Tilde{\mathbf{c}} = \mathbf{c}$. Therefore, through (i) we conclude that $\limsup_{n\rightarrow\infty} (\Tilde{\mu}_{\mathcal{I}^*}^{(n)})^T (\Tilde{\mathbf{A}}_{\mathcal{I}^*}^{(n)})^{-1} \mathbf{b} = \mu_{\mathcal{I}^*}^T \mathbf{A}_{\mathcal{I}^*}^{-1} \mathbf{b}$, i.e. the optimal value of \eqref{eq:SFLP}. Therefore, \eqref{eq:xinfinityoptimal} means that $\mathbf{x}^{(\infty)}$ is also an optimal solution of \eqref{eq:SFLP}, which contradicts with the uniqueness assumption. (Let $i\in \mathcal{I}^*\backslash \mathcal{J}$, we have $\mathbf{x}_i^{(\infty)} = 0 \neq \mathbf{x}_i^*$ and hence $\mathbf{x}^{(\infty)}\neq \mathbf{x}^*$.)

\subsection{Proof of Theorem \ref{thm:main}}\label{app:thm:main}
The proof follows the general strategy first introduced in \cite{audibert2010best} for analyzing fixed-budget BAI algorithms that reject arms successively. As with \cite{audibert2010best}, we assume that an infinite reward and cost sequence for each unknown arm $a\in [K_0]$ is drawn before the algorithm started. In this way, the empirical mean reward or cost of arm $a$ after $m$ draws is always well-defined.
    
Let $\armset_k$ denote the set of remaining arms after $k-1$ arms (including virtual arms) are eliminated. Recall that $(k)$ denotes the arm (including virtual arms) with the $k$-th smallest $\TD_a$ among all arms. At least one of the arms $a\in\{(K+L-k+1),\cdots,(K+L)\}$ is in $\armset_k$. If one of the arms in $\mathcal{I}^*$ is eliminated at the end of round $k$ for the first time, it implies that the following event $\mathcal{E}_k$ happened:
\begin{equation}\label{eq:orderviolation}
    \mathcal{I}^*\subset \armset_k, \quad\exists a \in \{(K+L-k+1),\cdots,(K+L)\}\cap \armset_k,  \quad \hat{\rho}_a^k\geq \min_{i\in\mathcal{I}^*} \hat{\rho}_i^k
\end{equation}
where $\hat{\rho}_i^k$ is the intersection value score for arm $i$ at the end of round $k$.

Next, fix $k$. Let $\hat{\mathbf{r}}\in\mathbb{R}^{K_0}, \hat{\mathbf{c}}\in \mathbb{R}^{L\times K_0}$ be the empirical means of rewards and costs respectively after each unknown arm has been drawn $n_k$ times. Let $\hat{\mathbf{A}}\in\mathbb{R}^{(L+1)\times (K+L)} $ denote the empirical version of the $\mathbf{A}$ matrix where $\mathbf{A}_{l, i} = c_{l,i}$ is replaced by $\hat{c}_{l, i}$ for $l\in [L], i\in [K_0]$.

If \eqref{eq:orderviolation} happens, at least one of the following events $\mathcal{E}_{k, 0}, (\mathcal{E}_{k, a})_{a \in\{(K+L-k+1), \cdots, (K+L) \}}$ must happen:
\begin{align}
    \mathcal{E}_{k, 0} &:= \{\det(\hat{\mathbf{A}}_{\mathcal{I}^*}) = 0, \text{ or } \hat{\mathbf{A}}_{\mathcal{I}^*}^{-1} \mathbf{b} \not\geq 0 \}\\
    \mathcal{E}_{k, a} &:= \{\mathcal{I}^*\subset \armset_k, a\in\armset_k,\hat{\mathbf{A}}_{\mathcal{I}^*}^{-1} \mathbf{b}\geq 0,~\hat{\rho}_a^k \geq \min_{i\in\mathcal{I}^*} \hat{\rho}_i^k \}
\end{align}

    On event $\mathcal{E}_{k, 0}$, by definition of $\Delta_0$ in \eqref{eq:Delta0}, we have
    \begin{equation}
        \xi_0:=\sigma_c^{-2}\sum_{i\in\mathcal{I}^*\cap [K_0] } \sum_{l=1}^L(c_{l,i} - \hat{c}_{l,i})^2 \geq \Delta_0^2
    \end{equation}
    
    For each $(l, i) \in [L]\times[K_0]$, the random variables $\sqrt{n_k} \sigma_c^{-1}(\hat{c}_{l, i} - c_{l, i})$ are i.i.d. standard normal random variables. Therefore, $n_k \xi_0$ is a chi-square random variable with degree $m = L\cdot|\mathcal{I}^*\cap [K_0]|$ (written as $\chi_m^2$). We have
    \begin{align}
        \Pr(\mathcal{E}_{k, 0}) &\leq \Pr(\chi_m^2 \geq n_k \Delta_0^2)\leq 3^{m/2}\exp\left(-\dfrac{n_k\Delta_0^2}{3}\right) \leq 3^{L(L+1)/2}\exp\left(-\dfrac{n_k\Delta_0^2}{3}\right)
    \end{align}
    where in the second inequality we applied Lemma \ref{lem:chi2}, and in the third inequality we used the fact that $m\leq L|\mathcal{I}^*| = L(L + 1)$.
    
    Now consider the event $\mathcal{E}_{k, a}$ for some $a \in\{(K+L-k+1), \cdots, (K+L) \}$. On this event, $\mathcal{I}^*$ corresponds to a BFS of $(\hat{\mathbf{A}}_{\armset_k}, \mathbf{b})$. Then, by definition of the IV scoring function \eqref{eq:IVscore}, for all $i\in\mathcal{I}^*$, we have $\hat{\rho}_i^k \geq \hat{\mu}_{\mathcal{I}^*}^T \hat{\mathbf{A}}_{\mathcal{I}^*}^{-1}\mathbf{b}$. Therefore, on event $\mathcal{E}_{k, a}$, we have $\hat{\rho}_a^k \geq \hat{\mu}_{\mathcal{I}^*}^T \hat{\mathbf{A}}_{\mathcal{I}^*}^{-1}\mathbf{b}$. Again, by definition of the scoring function, this means that there exists a basis $\mathcal{J}\subset [K+L], |\mathcal{J}| = L + 1$ such that $a\in\mathcal{J}$ and
    \begin{equation}
        \hat{\mathbf{A}}_{\mathcal{J}}^{-1} \mathbf{b}\geq 0,\quad \hat{\mu}_{\mathcal{J}}^T \hat{\mathbf{A}}_{\mathcal{J}}^{-1}\mathbf{b} \geq \hat{\mu}_{\mathcal{I}^*}^T \hat{\mathbf{A}}_{\mathcal{I}^*}^{-1}\mathbf{b}.
    \end{equation}
    
    Therefore, by the definition of $\TD_a$ in \eqref{eq:Deltaa}, the above implies that
    \begin{equation}
    \xi_{\mathcal{J}}:= \sum_{i\in(\mathcal{I}^*\cup\mathcal{J})\cap[K_0] }\left[ \sigma_r^{-2}(\hat{r}_i - r_i)^2 + \sigma_c^{-2}\sum_{l=1}^L (\hat{c}_{l, i} - c_{l, i})^2 \right] \geq \Delta_a^2.
    \end{equation}

    The random variables $\sqrt{n_k} \sigma_r^{-1}(\hat{r}_{i} - r_{i}), i\in[K_0]$ along with the random variables $\sqrt{n_k} \sigma_c^{-1}(\hat{c}_{l, i} - c_{l, i}), (l, i) \in [L]\times[K_0]$ are i.i.d. standard normal random variables. Therefore, $n_k \xi_{\mathcal{J}}$ is a chi-square random variable with degree $m_{\mathcal{J}} = (L+1)\cdot|(\mathcal{J}\cup \mathcal{I}^*)\cap [K_0]|$. We have $m_{\mathcal{J}} \leq (L+1)(|\mathcal{J}| + |\mathcal{I}^*|) = 2(L+1)^2$. Subsequently,
    \begin{align}
        \Pr(\mathcal{E}_{k, a}) &\leq \sum_{\mathcal{J}: |\mathcal{J}| = L + 1 , a\in \mathcal{J} } \Pr(\xi_{\mathcal{J}} \geq \Delta_a^2) \leq \sum_{\mathcal{J}: |\mathcal{J}| = L + 1 , a\in \mathcal{J} } \Pr(\chi_{m_{\mathcal{J}}}^2 \geq n_k\Delta_a)\\
        &\leq \binom{K+L-1}{L} 3^{(L+1)^2} \exp\left(-\dfrac{n_k\Delta_a^2}{3} \right)\tag{Lemma \ref{lem:chi2}, and $m_{\mathcal{J}}\leq 2(L+1)^2$}
    \end{align}

Therefore,
\begin{align}
    \Pr(\mathcal{E}_k) &\leq \Pr(\mathcal{E}_{k, 0}) + \sum_{a\in \{(K+L-k+1), \cdots, (K+L)\} }\Pr(\mathcal{E}_{k, a}) \\
    &\leq 3^{L(L+1)/2} \exp\left(-\dfrac{1}{3}n_k \Delta_0^2\right) + k \binom{K+L-1}{L} 3^{(L+1)^2} \exp \left(-\dfrac{1}{3}n_k \Delta_{(K+L+1-k)}^2 \right)
\end{align}

Finally, taking union bound and using $n_k \geq \frac{1}{\Psi(K_0, L)} \frac{N-K_0}{K+1-k}$, we have
    \begin{align}
        &\quad~\Pr(\armset^{\fcnm{}} \neq \mathcal{I}^*) \leq \sum_{k=1}^{K-1} \Pr(\mathcal{E}_k)\\
        &\leq \sum_{k=1}^{K-1}\left[ 3^{L(L+1)/2}\exp\left(-\dfrac{1}{3}n_k \Delta_0^2\right) + k \binom{K+L-1}{L} 3^{(L+1)^2} \exp \left(-\dfrac{1}{3}n_k \Delta_{(K+L+1-k)}^2 \right)\right] \\
        &\leq (K-1) 3^{L(L+1)/2} \exp\left(-\dfrac{(N-K_0)\Delta_0^{2}}{3(L/2 + \log K_0) K }\right) \\
        &\quad~+ \dfrac{K(K+1)}{2} \binom{K+L-1}{L} 3^{(L+1)^2}\exp\left(-\dfrac{N-K_0}{3(L/2 + \log K_0)} \min_{2\leq i\leq K} \dfrac{\Delta_{(L+i)}^2}{i} \right)
    \end{align}
    where in the last inequality we used
    \begin{align}
        n_k &\geq \dfrac{N-K_0}{\Psi(K_0, L)} \dfrac{1}{K+1-k} \qquad \forall k\in [K-1],\\
        \Psi(K_0, L) &\leq \dfrac{L+1}{2} + \sum\limits_{2 \le j \le  K_0 - L} \dfrac{1}{j} \leq \dfrac{L+1}{2} + \int_1^{K_0} \dfrac{1}{t}\mathrm{d} t = \dfrac{L+1}{2} + \log K_0
    \end{align}
 % This is correct. 
\subsection{Proof of Theorem \ref{prop:lowerbound}}\label{app:lowerbound}
Let $\theta$ satisfy Assumption \ref{assump:unique}. Consider an alternative (not necessarily feasible) instance $\theta'\in \varTheta$ with $\mathcal{I}^{\theta} \neq \mathcal{I}^{\theta'}$. By the consistency of $\pi_N$ we have
\begin{align}
    q_N'&:=\Pr_{\theta'}(\mathcal{X}^{\pi_N} \neq \mathcal{I}^{\theta})\xrightarrow{N\rightarrow\infty}1\label{eq:prooflb:qNprime}\\
    q_N&:=\Pr_{\theta}(\mathcal{X}^{\pi_N} \neq \mathcal{I}^{\theta})\xrightarrow{N\rightarrow\infty}0\label{eq:prooflb:qN}
\end{align}

Let $M_{a, N}$ denote the random number of times arm $a$ is pulled under algorithm $\pi_N$. Through Lemma 1 of \cite{kaufmann2016complexity} (which forms the foundation of numerous lower bounds related to bandit problems in the literature), we have
\begin{equation}\label{eq:prooflb:dpineq}
    \sum_{a=1}^{K} \E_{\theta'}[M_{a, N}] \mathbf{D}_{\mathrm{KL}}(\theta_a', \theta_a) \geq d_{\mathrm{KL}}(q_N', q_N)
\end{equation}
where $\mathbf{D}_{\mathrm{KL}}(\theta_a', \theta_a)$ is the KL divergence between the distributions of reward-cost vectors of arm $a$ under instances $\theta'$ and $\theta$. 

The RHS of \eqref{eq:prooflb:dpineq} satisfies
\begin{align}
    d_{\mathrm{KL}}(q_N', q_N) &= q_N'\log\left(\dfrac{q_N'}{q_N}\right) + (1-q_N')\log\left(\dfrac{1-q_N'}{1-q_N} \right) \\
    &\geq q_N'\log\left(\dfrac{1}{q_N}\right) - \log 2
\end{align}
where we have used the fact that $-z\log z - (1-z)\log(1-z)\leq \log 2$ for $z\in (0, 1)$.

Putting everything together and rearranging the terms, we have
\begin{align}
    \dfrac{1}{N}\log\left(\dfrac{1}{q_N}\right) \leq \dfrac{1}{q_N'}\left(\dfrac{\log 2}{N} + \sum_{a=1}^{K} \dfrac{\E_{\theta'}[M_{a, N}]}{N} \mathbf{D}_{\mathrm{KL}}(\theta_a', \theta_a) \right)\label{eq:prooflb:nonasymp}
\end{align}

Taking the limits on both sides, bounding $\frac{\E_{\theta'}[M_{a, N}]}{N}$ by 1, we have
\begin{equation}\label{eq:prooflb:boundbyKL}
    \limsup_{N\rightarrow\infty} -\dfrac{1}{N}\log \Pr_{\theta}(\mathcal{X}^{\pi_N} \neq \mathcal{I}^{\theta}) \leq \sum_{a=1}^{K}  \mathbf{D}_{\mathrm{KL}}(\theta_a', \theta_a)
\end{equation}

In particular, when $\pi_N$ is the \texttt{USLP} algorithm, \eqref{eq:prooflb:nonasymp} yields
\begin{equation}\label{eq:prooflb:boundbyKLUSLP}
    \limsup_{N\rightarrow\infty} -\dfrac{1}{N}\log \Pr_{\theta}(\mathcal{X}^{\mathrm{USLP}_N} \neq \mathcal{I}^{\theta}) \leq \dfrac{1}{K}\sum_{a=1}^{K}  \mathbf{D}_{\mathrm{KL}}(\theta_a', \theta_a)
\end{equation}

Using the formula for KL divergence between two multivariate Gaussian distributions, we have
\begin{align}
    \mathbf{D}_{\mathrm{KL}}(\theta_a', \theta_a) = \dfrac{1}{2}\left[\sigma_r^{-2}(r_a' - r_a)^2 + \sigma_c^{-2}\sum_{l=1}^L(c_{l,a}' - c_{l,a})^2\right]\label{eq:prooflb:KL}
\end{align}

Combining the above together and taking infimum over $\theta'\in\varTheta$, we have
\begin{align}
    \limsup_{N\rightarrow\infty} -\dfrac{1}{N}\log \Pr_{\theta}(\mathcal{X}^{\pi_N} \neq \mathcal{I}^{\theta}) &\leq \dfrac{1}{2}\inf_{\theta'\in \Theta}\left\{\sigma_r^{-2}\|\mathbf{r}'-\mathbf{r}\|_2^2 + \sigma_c^{-2}\|\mathbf{c}'-\mathbf{c}\|_2^2: \mathcal{I}^{\theta'} \neq \mathcal{I}^{\theta} \right\},\label{eq:prooflb:45}\\
    \limsup_{N\rightarrow\infty} -\dfrac{1}{N}\log \Pr_{\theta}(\mathcal{X}^{\mathrm{USLP}_N} \neq \mathcal{I}^{\theta}) &\leq \dfrac{1}{2K}\inf_{\theta'\in \Theta}\left\{\sigma_r^{-2}\|\mathbf{r}'-\mathbf{r}\|_2^2 + \sigma_c^{-2}\|\mathbf{c}'-\mathbf{c}\|_2^2: \mathcal{I}^{\theta'} \neq \mathcal{I}^{\theta} \right\},\label{eq:prooflb:46}
\end{align}
where $\|\cdot\|_2$ stands for the Euclidean 2-norm.

Note that $\mathcal{I}^{\theta'} \neq \mathcal{I}^{\theta}$ is true if either (i) $\mathcal{I}^{\theta}$ is an infeasible basis under $(\mathbf{A}', \mathbf{b})$; or (ii) $\mathcal{I}^{\theta}$ is feasible under $(\mathbf{A}', \mathbf{b})$ and there exists a basis set $\mathcal{J}\neq \mathcal{I}^{\theta}$ such that $(\mu_{\mathcal{I}^{\theta}}')^T(\mathbf{A}_{\mathcal{I}^{\theta}}')^{-1} \mathbf{b} \leq (\mu_{\mathcal{J}}')^T(\mathbf{A}_{\mathcal{J}}')^{-1} \mathbf{b}$. Therefore, 
\begin{align}
    &\quad~\inf_{\theta'\in \Theta}\left\{\sigma_r^{-2}\|\mathbf{r}'-\mathbf{r}\|_2^2 + \sigma_c^{-2}\|\mathbf{c}'-\mathbf{c}\|_2^2: \mathcal{I}^{\theta'} \neq \mathcal{I}^{\theta} \right\}\\
    &\leq \min\left\{\Delta_0^2, \min_{\mathcal{J}\neq \mathcal{I}^\theta}\Delta_{\mathcal{J}}^2 \right\} = \min\{\Delta_0^2, \Delta_{(L+2)}^2\}
\end{align}
concluding the proof.

In the last step above, there is a small caveat: the gaps $\Delta_0, (\Delta_{\mathcal{J}})_{\mathcal{J}\neq \mathcal{I}^{\theta}}$ were originally defined as infimums over all $\theta'\in \mathbb{R}^{(L+1)\times K}$ while the infimum in the RHS of \eqref{eq:prooflb:45} and \eqref{eq:prooflb:46} is taken over $\varTheta$, a proper subset of $\mathbb{R}^{(L+1)\times K}$. However, this is not a problem since $\varTheta$ is dense in $\mathbb{R}^{(L+1)\times K}$.

\textbf{Proof of $\varTheta$ being dense in $\mathbb{R}^{(L+1)\times K}$}: If $\theta'\in \mathbb{R}^{(L+1)\times K} \backslash \varTheta$ (i.e. $\theta'$ is a feasible instance that violates Assumption \ref{assump:unique}), then it is necessary that one of the following statements is true: (i) $\det(\mathbf{A}_{\mathcal{I}}') = 0$ for some basis $\mathcal{I}$; or (ii) $(\mu_{\mathcal{I}}')^T(\mathbf{A}_{\mathcal{I}}')^{-1} \mathbf{b} - (\mu_{\mathcal{J}}')^T(\mathbf{A}_{\mathcal{J}}')^{-1} \mathbf{b}  = 0$ for some bases $\mathcal{I}\neq \mathcal{J}$; or (iii) certain coordinate of $(\mathbf{A}_{\mathcal{I}}')^{-1} \mathbf{b}$ is zero for some basis $\mathcal{I}$. In either case, we have $h(\theta') = 0$ for some non-zero polynomial function $h:\mathbb{R}^{(L+1)\times K}\mapsto \mathbb{R}$. Since the set of zeroes of any non-zero polynomial function cannot contain any open ball, we conclude that $\mathbb{R}^{(L+1)\times K} \backslash \varTheta$ does not contain any open set, i.e. $\varTheta$ is dense in $\mathbb{R}^{(L+1)\times K}$.

\section{Additional Empirical Results}\label{app:moreexp}
In addition to the experiments in Section \ref{sec:empirical}, we also applied the three algorithms (\fcnm{}, \fcnmL{}, and \texttt{USLP}) to 6 instances with $L=2$ constraints. We set $K=K_0=24, \sigma_r = 1, \sigma_c = 0.5$ and $\costbound_1 = \costbound_2 = 1.0$. In all of the three instances, we set the costs of the 24 arms to be the 24 combinations of $c_{1, i}\in \{0.4, 0.6, 0.8, 1.0, 1.2, 1.4\}$ and $c_{2, i}\in \{0.7, 0.9, 1.1, 1.3\}$. Then, to define the rewards for each instance, we first pick a noise vector $(W_i)_{i=1}^{24}$ (which we will describe later) independently for each instance. In instance $D1$, we set $r_i = 1.0 - W_i$. In instance $D2$, we set $r_i = c_{1, i} - W_i$. In instance $D3$, we set $r_i = c_{1, i} + c_{2, i} - W_i$. We run the randomizations for a few times until the optimal support of each instance $Dj$ has exactly $j$ arms. Finally, we increment the reward of each arm in the optimal support by $0.02$ to ensure that the optimal support is unique and the instance is not overly difficult for any CBMAI algorithm. 

We consider two ways of choosing the random vector $(W_i)_{i=1}^{24}$: (i) a random permutation of $\{0.0, 0.02, \cdots, 0.46\}$. (ii) i.i.d. uniform random choices from $\{0.0, 0.02, \cdots, 0.28\}$. For the former choice, we will refer to the instance as $DjP$. For the latter, we will use $DjI$. The specific instances we used are reported in Table \ref{tab:D}.

For each combination of instance-algorithm-budget, we run the simulation for 5000 times independently and obtain the error rate as the proportion of times the algorithm output the wrong support. The results are provided in Figure \ref{fig:exp:D}. Each figure takes about 2 hours on an Apple M1 MacBook Air.

We can see that the \fcnmL{} algorithm on these two constraints instances either has the same performance as the \fcnm{} algorithm or does a bit better in terms of having a lower error rate.

\begin{figure}[!ht]
    \centering
    \includegraphics[width=0.85\textwidth]{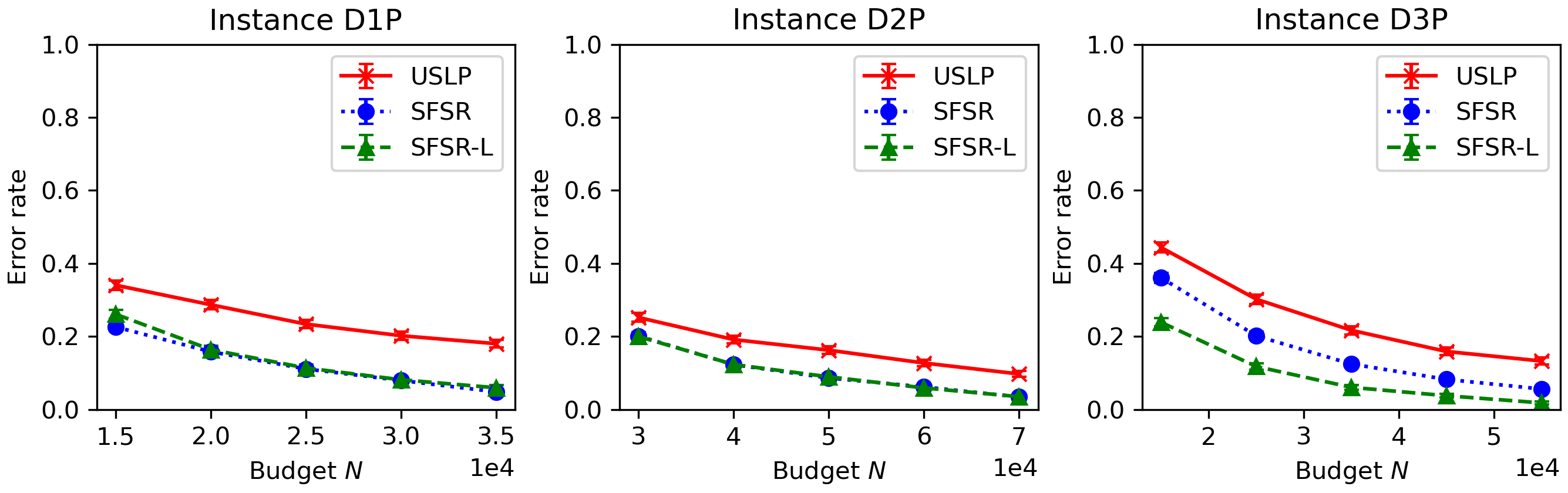}\\
    \includegraphics[width=0.85\textwidth]{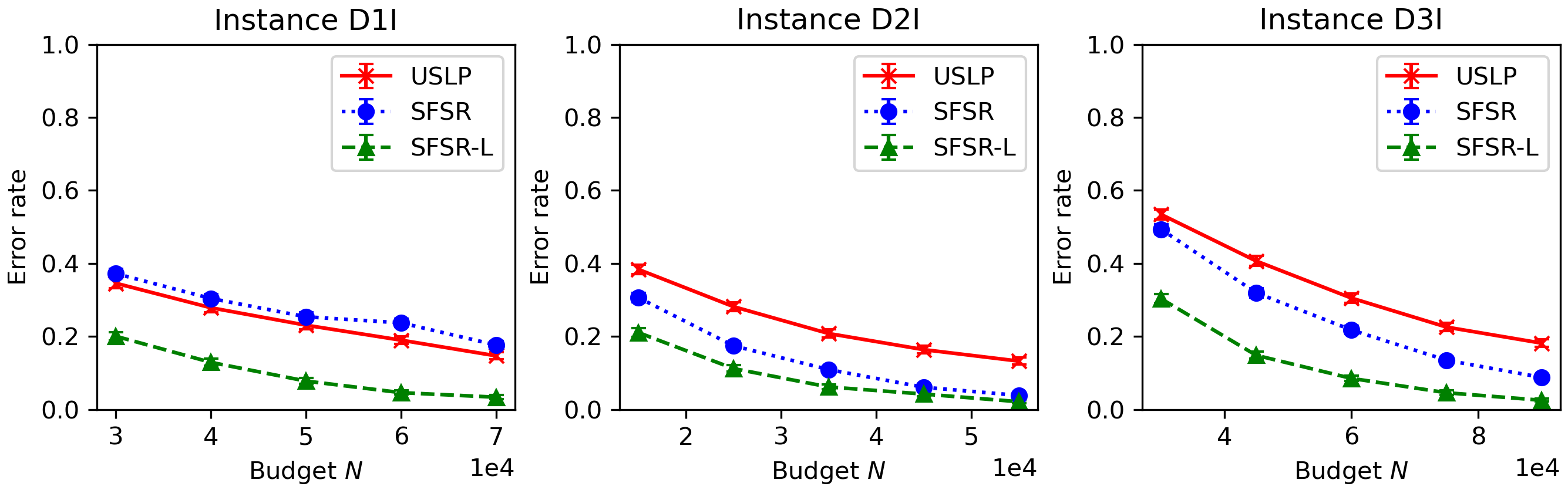}
    \caption{Simulation results for 6 instances with $L=2$ under varying budget. 95\% confidence intervals are indicated and tight.}
    \label{fig:exp:D}
\end{figure}

\begin{table}[!ht]
    \caption{Description of the mean rewards and costs of instances. The rewards of arms in the optimal support is shown in bold. Top Row (from left to right): $D1P,D2P,D3P$. Bottom Row (from left to right): $D1I,D2I,D3I$.}
    \centering\small
    \scalebox{0.85}{
        \begin{tabular}{ccccc}
        \hline
        \!\!$c_1\backslash c_2$\!\! & $0.7$ & $0.9$ & $1.1$ & $1.3$\\\hline\hline
        $0.4$ & $0.88$ & $0.80$ & $0.82$ & $0.66$\\\hline
        $0.6$ & $0.72$ & $\mathbf{1.02}$ & $0.70$ & $0.54$\\\hline
        $0.8$ & $0.92$ & $0.74$ & $0.94$ & $0.84$\\\hline
        $1.0$ & $0.76$ & $0.60$ & $0.56$ & $0.86$\\\hline
        $1.2$ & $0.98$ & $0.64$ & $0.68$ & $0.78$\\\hline
        $1.4$ & $0.62$ & $0.96$ & $0.90$ & $0.58$\\\hline
        \end{tabular}
    }
    \scalebox{0.85}{
        \begin{tabular}{ccccc}
        \hline
        \!\!$c_1\backslash c_2$\!\! & $0.7$ & $0.9$ & $1.1$ & $1.3$\\\hline\hline
        $0.4$ & $0.08$ & $0.28$ & \!$-0.02$\! & $0.22$\\\hline
        $0.6$ & $0.20$ & $0.46$ & $0.54$ & $0.40$\\\hline
        $0.8$ & $0.42$ & $0.52$ & $\mathbf{0.80}$ & $0.34$\\\hline
        $1.0$ & $0.92$ & $0.78$ & $0.96$ & $0.70$\\\hline
        $1.2$ & $0.94$ & $0.76$ & $1.10$ & $0.86$\\\hline
        $1.4$ & $\mathbf{1.42}$ & $1.16$ & $1.04$ & $1.24$\\\hline
        \end{tabular}
    }
    \scalebox{0.85}{
        \begin{tabular}{ccccc}
        \hline
        \!\!$c_1\backslash c_2$\!\! & $0.7$ & $0.9$ & $1.1$ & $1.3$\\\hline\hline
        $0.4$ & $1.04$ & $1.22$ & $1.28$ & $1.26$\\\hline
        $0.6$ & $0.98$ & $1.22$ & $1.60$ & $1.54$\\\hline
        $0.8$ & $1.26$ & $1.40$ & $\mathbf{1.88}$ & $1.84$\\\hline
        $1.0$ & $\mathbf{1.72}$ & $1.76$ & $1.64$ & $1.92$\\\hline
        $1.2$ & $1.78$ & $1.70$ & $1.96$ & $2.08$\\\hline
        $1.4$ & $1.94$ & $\mathbf{2.30}$ & $2.32$ & $2.50$\\\hline
        \end{tabular}
    }
    \\\vspace{.5em}
    \scalebox{0.85}{
        \begin{tabular}{ccccc}
            \hline
            \!\!$c_1\backslash c_2$\!\! & $0.7$ & $0.9$ & $1.1$ & $1.3$\\\hline\hline
            $0.4$ & $0.84$ & $\mathbf{1.02}$ & $0.74$ & $0.76$\\\hline
            $0.6$ & $0.84$ & $0.88$ & $0.96$ & $0.90$\\\hline
            $0.8$ & $0.90$ & $0.98$ & $0.92$ & $0.80$\\\hline
            $1.0$ & $0.98$ & $0.98$ & $0.90$ & $0.74$\\\hline
            $1.2$ & $0.94$ & $0.90$ & $0.88$ & $0.72$\\\hline
            $1.4$ & $0.78$ & $0.82$ & $0.88$ & $0.84$\\\hline
        \end{tabular}
    }
    \scalebox{0.85}{
        \begin{tabular}{ccccc}
        \hline
        \!\!$c_1\backslash c_2$\!\! & $0.7$ & $0.9$ & $1.1$ & $1.3$\\\hline\hline
        $0.4$ & $\mathbf{0.40}$ & $0.18$ & $0.28$ & $0.14$\\\hline
        $0.6$ & $0.44$ & $0.54$ & $0.40$ & $0.32$\\\hline
        $0.8$ & $0.56$ & $0.52$ & $0.68$ & $0.64$\\\hline
        $1.0$ & $0.84$ & $0.80$ & $0.82$ & $0.74$\\\hline
        $1.2$ & $0.94$ & $1.18$ & $1.02$ & $1.12$\\\hline
        $1.4$ & $\mathbf{1.42}$ & $1.16$ & $1.24$ & $1.24$\\\hline
        \end{tabular}
    }
    \scalebox{0.85}{
        \begin{tabular}{ccccc}
        \hline
        \!\!$c_1\backslash c_2$\!\! & $0.7$ & $0.9$ & $1.1$ & $1.3$\\\hline\hline
        $0.4$ & $0.92$ & $1.12$ & $1.32$ & $1.42$\\\hline
        $0.6$ & $1.14$ & $1.42$ & $1.62$ & $1.68$\\\hline
        $0.8$ & $1.30$ & $\mathbf{1.70}$ & $1.68$ & $\mathbf{2.06}$\\\hline
        $1.0$ & $1.46$ & $1.82$ & $2.02$ & $2.04$\\\hline
        $1.2$ & $1.84$ & $2.02$ & $2.12$ & $2.24$\\\hline
        $1.4$ & $2.06$ & $\mathbf{2.28}$ & $2.32$ & $2.58$\\\hline
        \end{tabular}
    }\\\vspace{1em}
    \label{tab:D}
\end{table}

\end{document}